\documentclass[11pt, letterpaper]{article}
\usepackage[margin=1in]{geometry}
\usepackage[font=small,labelfont=bf]{caption}
\usepackage{microtype}
\usepackage{graphicx}
\usepackage{subcaption}
\usepackage{booktabs} 
\usepackage{algorithm}
\usepackage{algpseudocode}
\usepackage{multirow}

\usepackage{amsmath}
\usepackage{amssymb}
\usepackage{mathtools}
\usepackage{amsthm}

\usepackage{natbib}

\usepackage[textsize=tiny]{todonotes}

\usepackage{xcolor}
\usepackage{bbm}

\usepackage[colorlinks=true]{hyperref}
\hypersetup{%
,urlcolor=blue
,citecolor=blue
,linkcolor=red
}

\newcommand{\Pxy}{P_{X,Y}}

\newcommand{\calA}{\mathcal{A}}
\newcommand{\calB}{\mathcal{B}}

\newcommand{\calD}{\mathcal{D}}

\newcommand{\calH}{\mathcal{H}}

\newcommand{\calM}{\mathcal{M}}

\newcommand{\calP}{\mathcal{P}}

\newcommand{\calR}{\mathcal{R}}
\newcommand{\calS}{\mathcal{S}}
\newcommand{\calT}{\mathcal{T}}

\newcommand{\calX}{\mathcal{X}}
\newcommand{\calY}{\mathcal{Y}}

\newcommand{\bx}{\mathbf{x}}
\newcommand{\by}{\mathbf{y}}

\newcommand{\bg}{\mathbf{g}}
\newcommand{\bw}{\mathbf{w}}

\newcommand{\bz}{\mathbf{z}}
\newcommand{\bp}{\mathbf{p}}

\newcommand{\bq}{\mathbf{q}}
\newcommand{\be}{\mathbf{e}}

\newcommand{\bW}{\mathbf{W}}

\newcommand{\bP}{\mathbf{P}}
\newcommand{\bZ}{\mathbf{Z}}

\newcommand{\R}{\mathbb{R}}

\newcommand{\E}{\mathbb{E}}

\newtheorem{defn}{Definition}

\newtheorem{prop}{Proposition}

\DeclareMathOperator*{\argmax}{arg\,max}
\DeclareMathOperator*{\argmin}{arg\,min}

\usepackage{tikz}
\usetikzlibrary{arrows}

\newcommand{\vol}{\mathsf{vol}}
\newcommand{\rad}{\mathsf{rad}}
\newcommand{\diam}{\mathsf{diam}}

\newcommand{\diag}{\mathsf{diag}}
\newcommand{\convexhull}{\mathsf{convex\\ hull}}

\begin{document}

\title{Rashomon Capacity: A Metric for\\Predictive Multiplicity in Classification}
\date{}
\author{
    Hsiang~Hsu,
    and~Flavio~du~Pin~Calmon\thanks{H. Hsu and F. P. Calmon are with the John A. Paulson School of Engineering and Applied Sciences, Harvard University, Boston, MA 02134. Emails: \texttt{hsianghsu@g.harvard.edu, flavio@seas.harvard.edu}.This material is based upon work supported by the National Science Foundation under grants CAREER 1845852, IIS 1926925, and FAI 2040880, and by Meta Ph.D. fellowship.}
}

\maketitle

\begin{abstract}
Predictive multiplicity occurs when classification models with statistically indistinguishable performances assign conflicting predictions to individual samples. When used for decision-making in applications of consequence (e.g., lending, education, criminal justice), models developed without regard for predictive multiplicity may result in unjustified and arbitrary decisions for specific individuals. We introduce a new metric, called Rashomon Capacity, to measure predictive multiplicity in probabilistic classification. Prior metrics for predictive multiplicity focus on classifiers that output thresholded (i.e., 0-1) predicted classes. In contrast, Rashomon Capacity applies to probabilistic classifiers, capturing more nuanced score variations for individual samples. We provide a rigorous derivation for Rashomon Capacity, argue its intuitive appeal, and demonstrate how to estimate it in practice. We show that Rashomon Capacity yields principled strategies for disclosing conflicting models to stakeholders. Our numerical experiments illustrate how Rashomon Capacity captures predictive multiplicity in various datasets and learning models, including neural networks. The tools introduced in this paper can help data scientists measure and report predictive multiplicity prior to model deployment.
\end{abstract}


\section{Introduction}\label{sec:intro}
\emph{Rashomon effect}, introduced by \citet{breiman2001statistical}, describes the phenomenon where a multitude of distinct predictive models achieve similar training or test loss.
Breiman reported observing the Rashomon effect in several model classes, including linear regression, decision trees, and small neural networks.
In a foresighted experiment, Breiman noted that, when retraining a  neural network 100 times on  three-dimensional data with different random initializations, he ``\emph{found 32 distinct minima, each of which gave a different picture, and having about equal test set error}'' \citep[Section~8]{breiman2001statistical}.
The set of almost-equally performing models for a given learning problem is called the \emph{Rashomon set} \citep{fisher2019all, semenova2019study}.

We focus on a facet of the Rashomon effect in classification problems called \emph{predictive multiplicity}. 
Predictive multiplicity occurs when competing models in the Rashomon set assign conflicting predictions to individual samples \citep{marx2020predictive}. 
Fig.~\ref{fig:three-examples} presents an updated version of Breiman's neural network experiment and illustrates predictive multiplicity in three classification tasks with different data domains and neural network architectures. 
Here, models that achieve statistically-indistinguishable performance on a test set assign wildly different predictions to an input sample. 
If predictive multiplicity is not accounted for, the output for this sample may  ultimately depend on arbitrary choices made during training (e.g., parameter initialization).

Predictive multiplicity captures the potential individual-level harm introduced by an arbitrary choice of a single model in the Rashomon set.
When such a model is used to support automated decision-making in sectors dominated by a few companies or Government---labeled \emph{Algorithmic Leviathans} in \citet[Section 3]{creel2021algorithmic}---predictive multiplicity can lead to unjustified and systemic exclusion of individuals from critical opportunities. 
For example, an algorithm used for lending may deny a loan to a specific applicant. However, during model development, there may have been a competing model which performs equally well on average, yet would have approved the loan for this individual. 
As another example, Governments are increasingly turning to algorithms for grading exams that grant access to higher-level education (see, e.g., UK \citep{smith2020algorithmic} and Brazil \citep{enem2020}). 
Here, again, accounting for predictive multiplicity is critical: an arbitrary choice of a single model in the Rashomon set may lead to an unwarranted restriction of educational opportunities to an individual student.

\begin{figure}[!tb]
     \centering
     \begin{subfigure}[b]{0.32\textwidth}
         \centering
         \includegraphics[width=\textwidth]{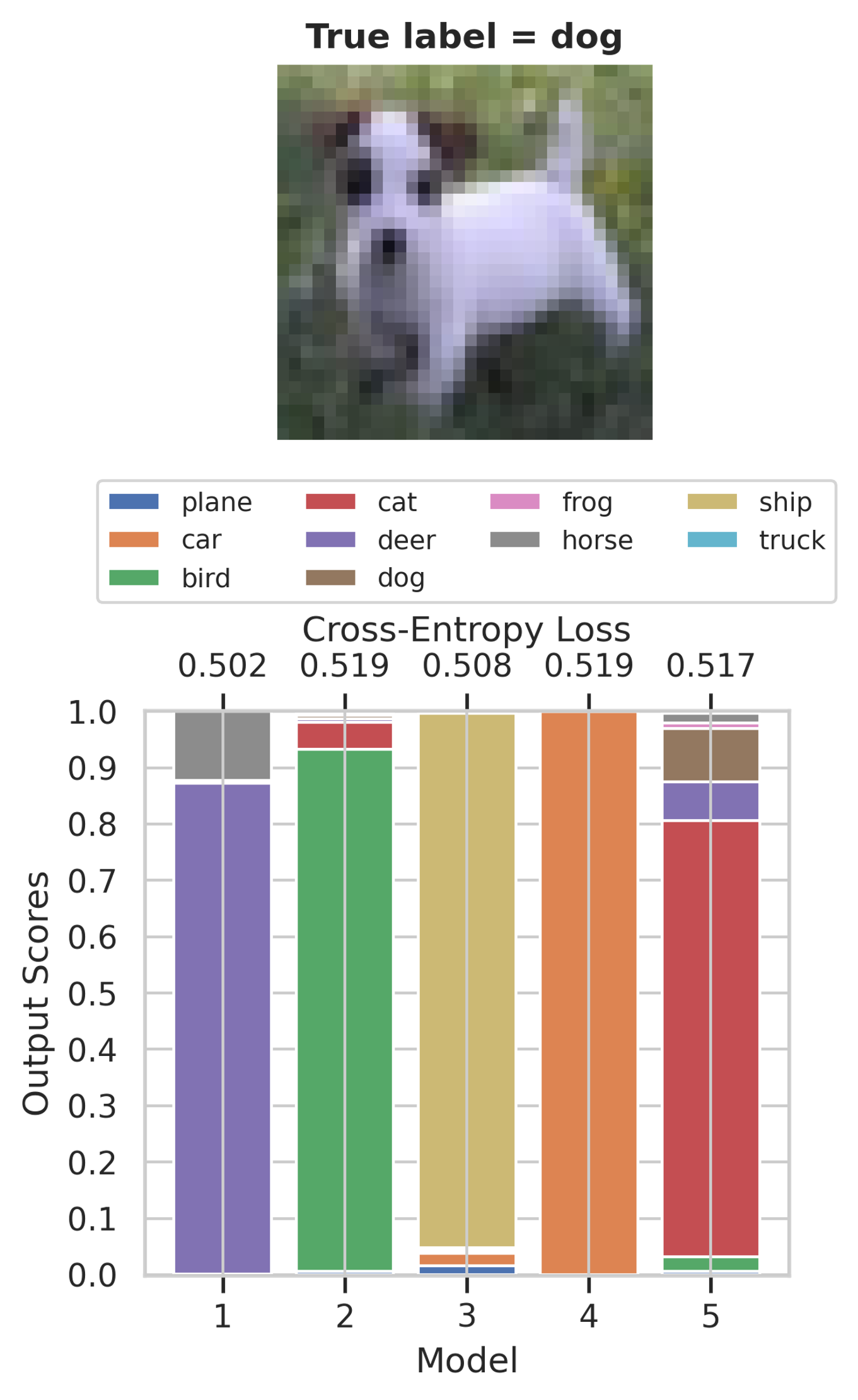}
         \caption{\small CIFAR-10 dataset}
         \label{fig:intro-cifar10}
     \end{subfigure}
     \begin{subfigure}[b]{0.3\textwidth}
         \centering
         \includegraphics[width=\textwidth]{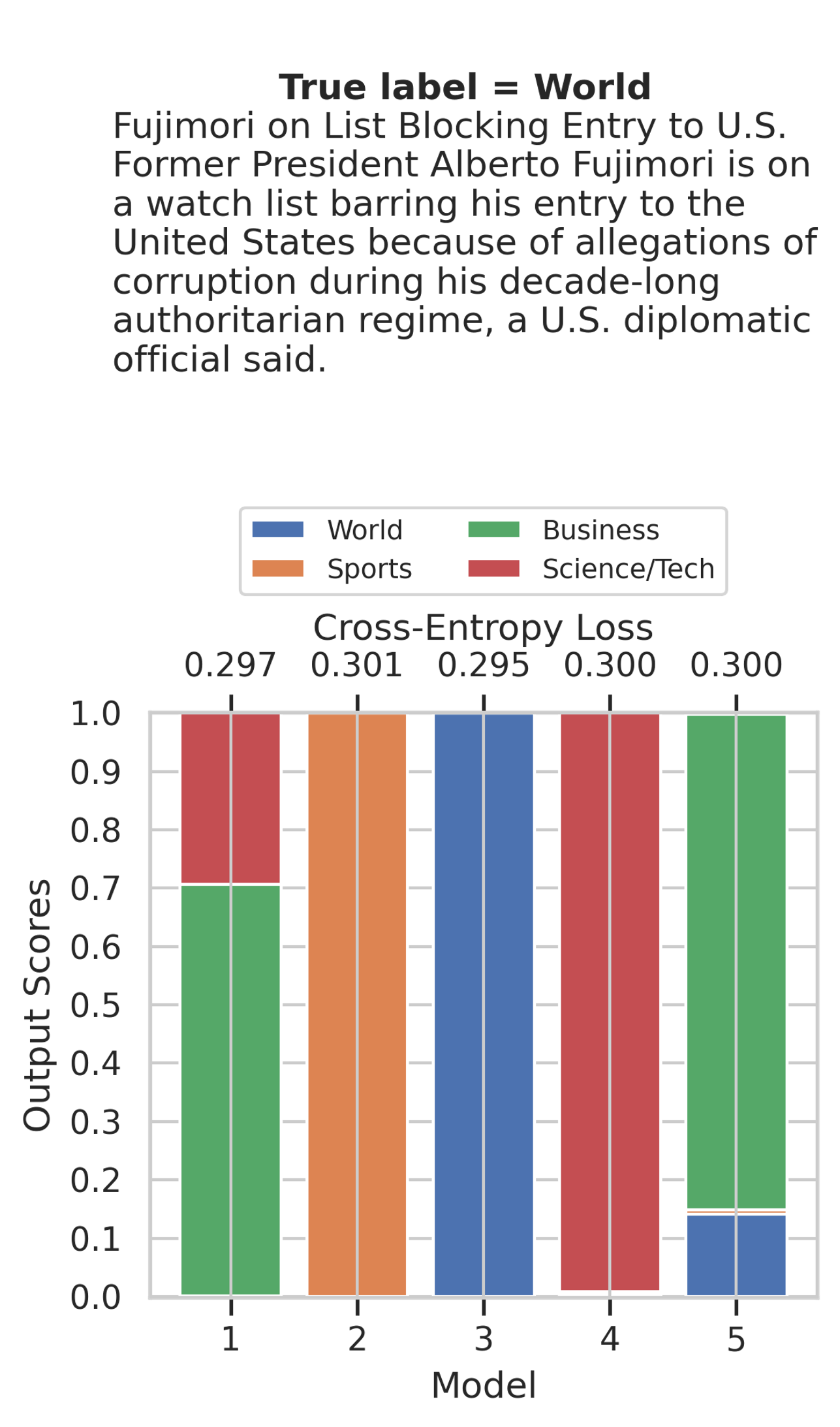}
         \caption{\small AG News dataset}
         \label{fig:intro-agnews}
     \end{subfigure}
     \begin{subfigure}[b]{0.305\textwidth}
         \centering
         \includegraphics[width=\textwidth]{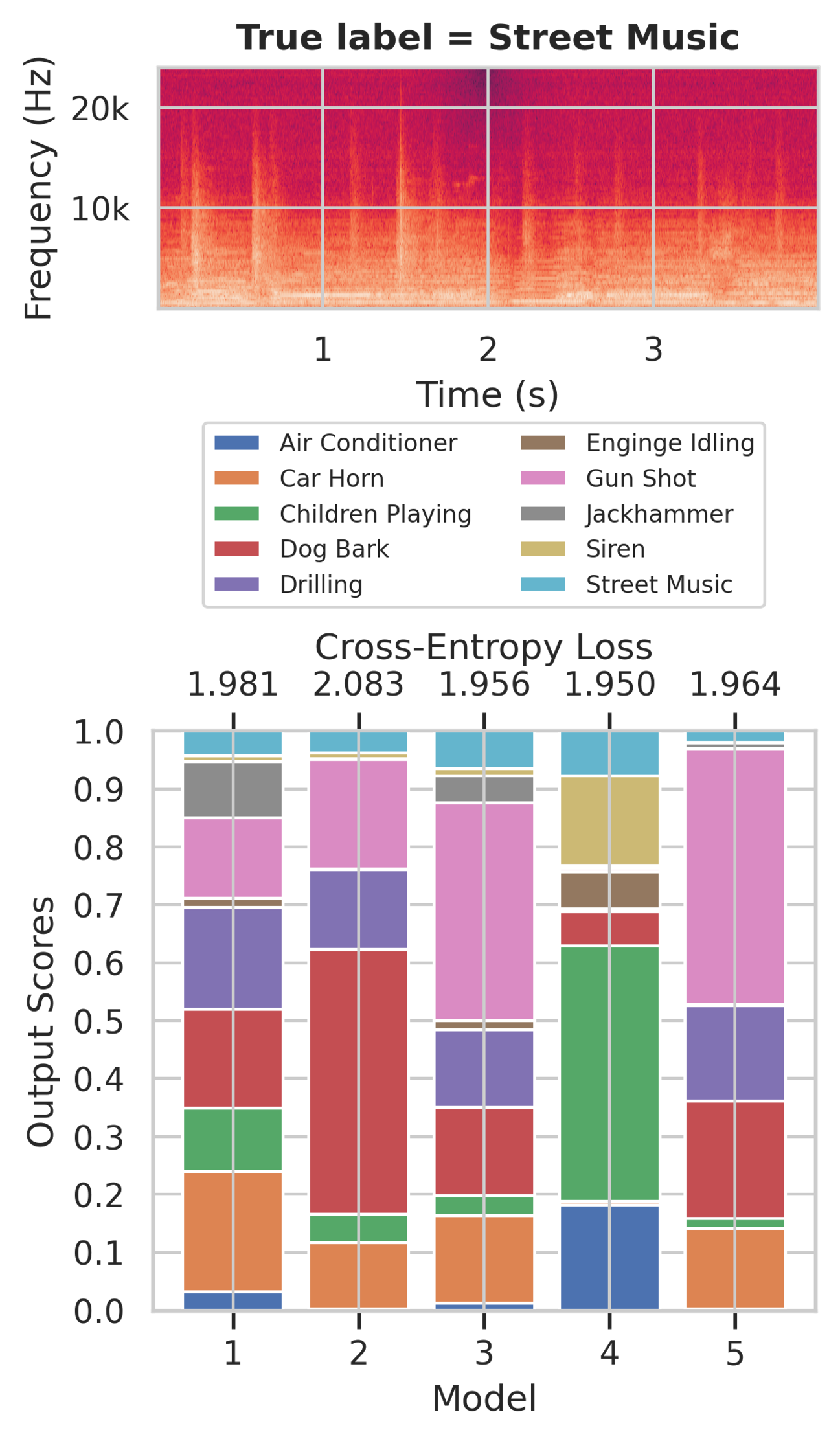}
         \caption{\small UrbanSound8k dataset.}
         \label{fig:intro-score}
     \end{subfigure}
        \caption{The scores (bottom) of a sample (top) generated by competing models. Predictive multiplicity occurs on different  data domains and learning models, including an image dataset (CIFAR-10 \citep{krizhevsky2009learning}) trained with VGG16 \citep{simonyan2014very}, a natural language dataset (AG News \citep{zhang2015character}) trained with a simple neural networks after tokenization, and an audio dataset (UrbanSound8k \citep{salamon2014dataset}) trained with LSTM \citep{greff2016lstm}.}
        \label{fig:three-examples}
\end{figure}

We introduce new methods for measuring and reporting predictive multiplicity in probabilistic classification.
First, we postulate several properties that a predictive multiplicity metric must satisfy to simplify its interpretation by stakeholders. 
We then provide a new predictive multiplicity metric called \emph{Rashomon Capacity}. 
Rashomon Capacity quantifies score variations among  models in the Rashomon set for a given input sample. 
Unlike prior metrics restricted to thresholded scores (i.e., decisions),  Rashomon Capacity can be applied to  probabilistic classifiers that output a probability distribution over a set of classes (e.g.,  neural networks with soft-max output layers). 
Communicating such score disagreements helps stakeholders understand whether a given prediction is arbitrary, e.g., depending on randomness during training rather than  patterns in the data.

We show that Rashomon Capacity of an input sample can be entirely captured by at most $c$ models in the Rashomon set, where $c$ is the number of predicted classes, \emph{regardless of the size of the Rashomon set}.
Remarkably, the computation of Rashomon Capacity also sheds light on a strategy for resolving predictive multiplicity. 
Instead of releasing a single model, we provide a greedy algorithm for identifying \emph{a subset of models} in the Rashomon set that captures most of the score variations across a dataset. 
These models can be communicated to a stakeholder, empowering them to decide how to resolve conflicting scores via, e.g., randomization \citep{creel2021algorithmic} and bagging \citep{breiman2001statistical}. 
In summary, our main contributions include: 
\begin{enumerate}
    \item We postulate desirable properties that \emph{any} predictive multiplicity metric must satisfy. These properties motivate our definition of Rashomon Capacity and provide guidelines for the creation of new multiplicity metrics in future research. We also outline computational challenges in estimating predictive multiplicity in practice.
    \item We introduce a new score-based metric for \emph{quantifying} predictive multiplicity  called Rashomon Capacity. 
    Rashomon Capacity can be applied to measure score variations across competing classifiers that output either  raw or thresholded (i.e. 0-1) scores.
    \item We describe a methodology for \emph{reporting} predictive multiplicity in probabilistic classification using Rashomon Capacity, with examples on different datasets and models. We advocate that predictive multiplicity must be reported to stakeholders in, for example, model cards \citep{mitchell2019model}.
    \item We propose a procedure for \emph{resolving}  predictive multiplicity in probabilistic classifiers. Even though the Rashomon set may span a large (potentially uncountable) number of models, we show that the score variation for a sample is fully captured by a small subset of models in the Rashomon set. Communicating these predictions to stakeholders can empower them to decide how to resolve predictive multiplicity. 
\end{enumerate}

Omitted proofs, additional explanations and discussions, details on experiment setups and training, and additional experiments are included in the Appendix.
Code to reproduce our experiments is available at \href{https://github.com/HsiangHsu/rashomon-capacity}{https://github.com/HsiangHsu/rashomon-capacity}.

\section{Background and related work}
\label{sec:background}
\textbf{Notation.} We consider a dataset $\calD = \{(\bx_i, \by_i)\}_{i=1}^n$, e.g., a training or test set, for a classification task with $c$ classes/labels.
Each sample pair $(\bx_i, \by_i)$ is drawn i.i.d. from $\Pxy$ with support $\calX\times\Delta_c$. 
Here, $\Delta_c \triangleq\{(r_1, \cdots, r_c)\in[0, 1]^c; \sum_{i=1}^c r_i = 1\}$ denotes the $c$-dimensional probability simplex.
Let $\be_k$ be a length-$c$ indicator vector with one in the ${k}^{th}$ position and zero elsewhere, i.e., $[\be_k]_k = 1$, and $[\be_k]_j = 0 \; \forall j \neq k$, where $[\cdot]_j$ denotes the ${j}^{th}$ entry of a vector.
Each $\by_i$ is one-hot encoded, i.e., $\by_i \in \{\be_k\}_{k=1}^c$.
$\mathbbm{1}(\cdot)$ denotes the indicator function.

We denote by $\calH$ a hypothesis space, i.e., a set of candidate probabilistic classifier is parameterized by $\theta \in \Theta \subseteq \R^d$ that approximate $P_{Y|X=\bx_i}$, i.e., $\calH \triangleq \{h_\theta: \calX \to \Delta_c: \theta \in \Theta\}$.
The loss function used to evaluate model performance is denoted by $\ell: \Delta_c\times\Delta_c\to\R^+$ (e.g., cross-entropy) and $L(h_\theta)\triangleq \E_{\Pxy}[\ell(h_\theta(X), Y)]$ the population risk.
As usual, the population risk is approximated by the empirical risk  $\hat{L}(h_\theta) \triangleq \frac{1}{n}\sum_{i=1}^n \ell(h_\theta(\bx_i), \by_i)$.

\paragraph{Rashomon set, Rashomon ratio, and pattern Rashomon ratio.} 
We define the Rashomon set as the set of all  models in the hypothesis space that yield similar average loss. 
Formally, given a Rashomon parameter $\epsilon \geq 0$, the Rashomon set is defined as an $\epsilon$-level set \citep{semenova2019study}:
\begin{equation}\label{eq:rashomon-set}
    \calR(\calH, \epsilon) \triangleq \{ h_\theta \in \calH; L(h_\theta) \leq \epsilon \}.
\end{equation}
Note that the Rashomon set is determined by the hypothesis space $\calH$, the Rashomon parameter $\epsilon$, and also implicitly by the data distribution due to the evaluation of $L(h_\theta)$.
The cardinality $|\calR(\calH, \epsilon)|$ or the volume\footnote{Since $\calH$ is parameterized, the volume of $\calR(\calH, \epsilon) = \{ \theta\in\Theta; L(h_\theta) \leq \epsilon \}$ can be directly computed in $\R^d$.} $\vol(\calR(\calH, \epsilon))$ of the Rashomon set (depending on whether the Rashomon set has finite elements) can be used to quantify the size of the Rashomon set.
Given $\calR(\calH, \epsilon)$, the \emph{Rashomon ratio} \citep[Defn.~2]{semenova2019study} is defined as $\hat{\calR}(\calH, \epsilon) \triangleq \frac{\vol(\calR(\calH, \epsilon))}{\vol(\calH)}$.
$\hat{\calR}(\calH, \epsilon)$ represents the fraction of models in the hypothesis space that fit the data about equally well. 
A large Rashomon ratio indicates high multiplicity. 
Moreover, models with various desirable properties, such as better generalizability, can often exist inside a large Rashomon set.
Similar to the Rashomon ratio, \emph{pattern Rashomon ratio} \citep[Defn.~12]{semenova2019study} is defined as the ratio of the count of all possible binary predicted classes given by the functions in the Rashomon set to that given by the functions in the hypothesis space.
The complexity of computing pattern Rashomon ratio grows exponentially with the number of samples, yielding an ``expensive'' metric for predictive multiplicity when applied to large datasets.

\paragraph{Ambiguity and discrepancy.}
Instead of characterizing multiplicity in the hypothesis/ parameter space, \citet{marx2020predictive} measure multiplicity in terms of the thresholded outputs (i.e., predicted classes) of a classifier, and propose two metrics: \emph{ambiguity} and \emph{discrepancy}.
Ambiguity is the proportion of samples in a dataset that can be assigned  conflicting predictions by  competing classifiers in the Rashomon set. Discrepancy is the maximum number of predictions  that could change in a dataset if we were to switch between models within the Rashomon set.
More precisely, given a base model $\hat{h}$, the ambiguity $\alpha_\epsilon(\hat{h})$ and the discrepancy $\delta_\epsilon(\hat{h})$ are respectively defined  as 
\begin{equation}\label{eq:ambiguity-discrepancy}
\begin{aligned}
\alpha_\epsilon(\hat{h}) &\triangleq \frac{1}{n} \sum_{i=1}^n \max\limits_{h \in \calR(\calH, \epsilon)} \mathbbm{1} \left[ \argmax h(\bx_i) \neq \argmax \hat{h}(\bx_i) \right], \\
\delta_\epsilon(\hat{h}) &\triangleq \max\limits_{h \in \calR(\calH, \epsilon)} \frac{1}{n} \sum_{i=1}^n  \mathbbm{1} \left[ \argmax h(\bx_i) \neq \argmax\hat{h}(\bx_i) \right].
\end{aligned}
\end{equation}
For linear classifiers, both quantities in \eqref{eq:ambiguity-discrepancy} can be estimated by mixed integer programming \citep[Section~3]{marx2020predictive}.
A small $\epsilon$ could still lead to a large ambiguity, see Appendix~\ref{appendix:small-epsilon} for a discussion.

\begin{figure}[t!]
    \centering
     \begin{subfigure}[b]{0.25\textwidth}
         \centering
         \includegraphics[width=\textwidth]{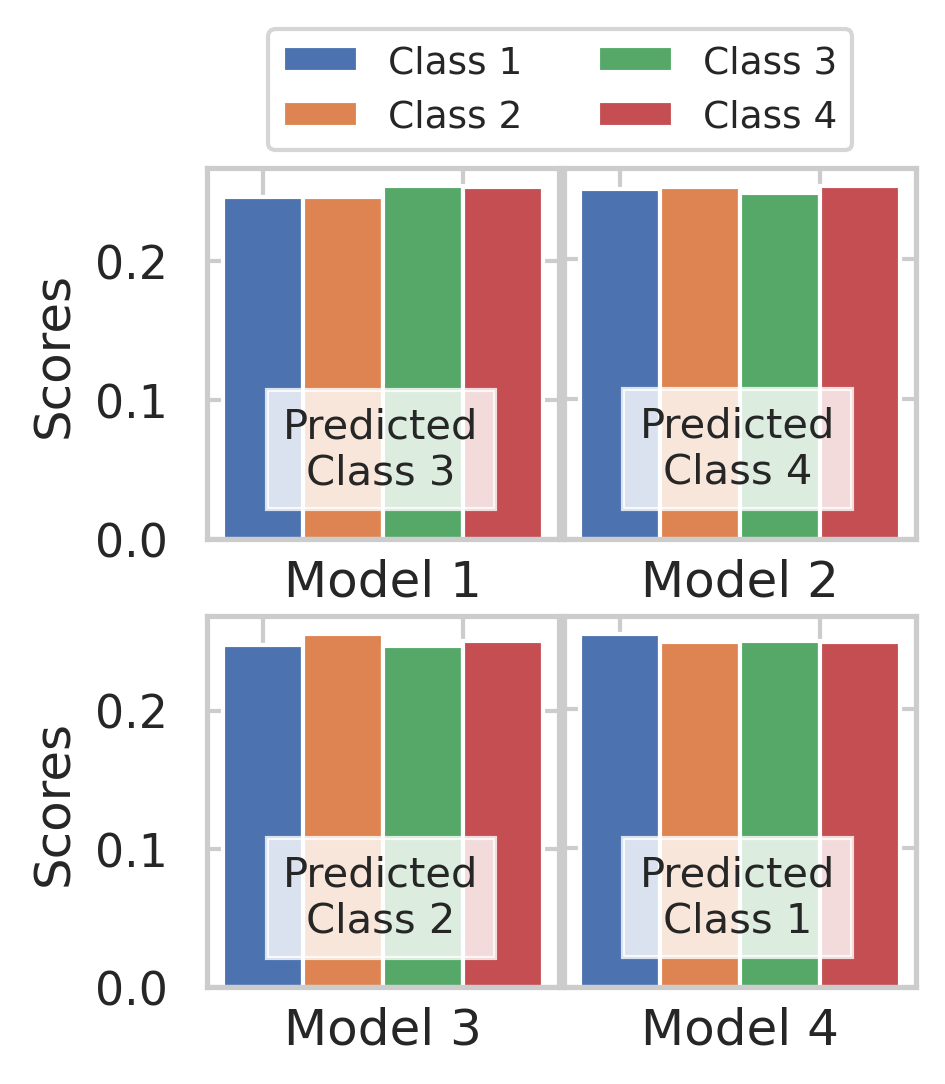}
     \end{subfigure}
     \hspace{2em}
     \begin{subfigure}[b]{0.65\textwidth}
         \centering
         \includegraphics[width=\textwidth]{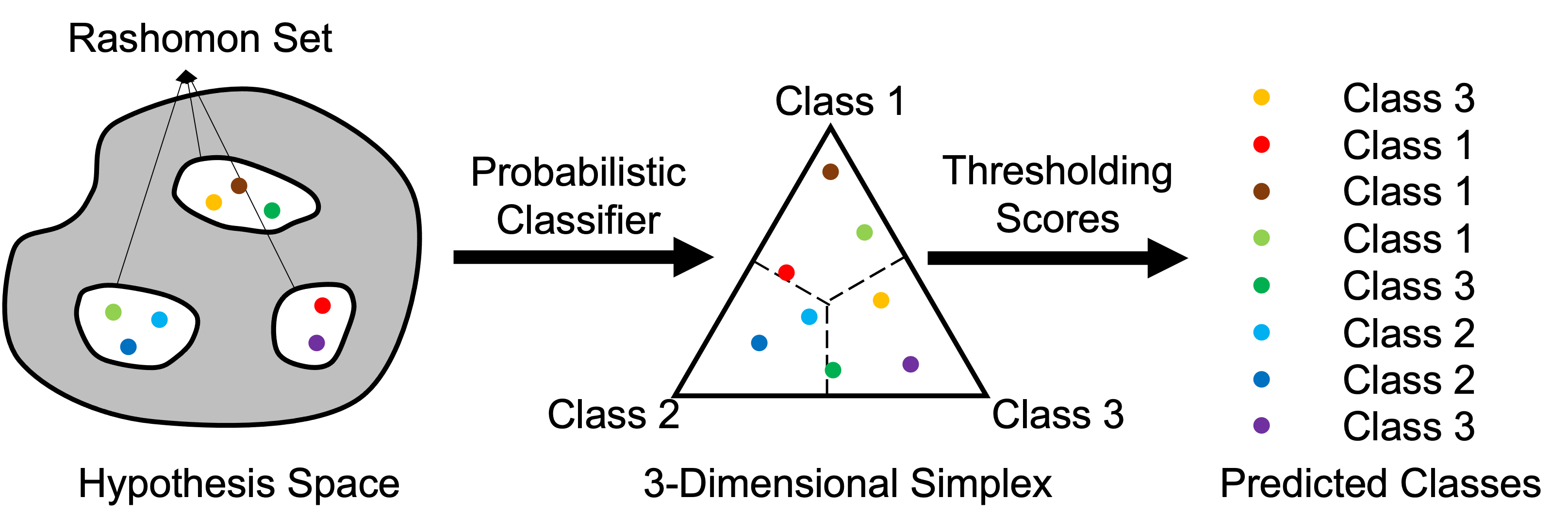}
     \end{subfigure}
     \caption{\textbf{Left}: 4 models output similar scores for  4 classes, but the predicted classes produced via thresholding are  very different. This may lead  to high ambiguity and/or discrepancy (cf.~\eqref{eq:ambiguity-discrepancy}). \textbf{Right}: Given a sample, the scores obtained from different models (colored dots) in the Rashomon set lead to different predicted classes by thresholding the scores (based on decision boundaries, i.e., dashed lines in the simplex). Prior work either measure multiplicity on the hypothesis set, e.g., the Rashomon ratio, or in terms of predicted classes, e.g., pattern Rashomon ratio and ambiguity/discrepancy. In this work, we measure predictive multiplicity directly in terms of scores on the probability simplex, i.e., either thresholded or raw scores.}
     \label{fig:score-based-metric}
\end{figure}

\section{Measuring predictive multiplicity of probabilistic classifiers}\label{sec:rashomon-capacity}
The metrics in \eqref{eq:ambiguity-discrepancy} for predictive multiplicity are based on predicted classes. 
Thus, they require finding the $\argmax$ or thresholding the scores at the output of a classifier. In probabilistic classification, thresholding may mask similar predictions produced by competing models and artificially increase multiplicity:  output scores can be almost equal across different classes, yet the (thresholded) predicted classes can be very different. 
For example, two scores $[0.49, 0.51]$ and $[0.51, 0.49]$ for a binary classification problem can lead to entirely different predicted classes after thresholding---1 and 0, respectively---and ultimately overestimate predictive multiplicity (see Fig.~\ref{fig:score-based-metric} (Left) for another multi-class example). In fact, predictive multiplicity metrics based on predicted classes may yield multiplicity \emph{even for a single fixed model} when, for example, the threshold criteria for output scores is changed. 
This subtle, yet important difference motivates us to reconsider existing metrics and introduce a new predictive multiplicity metric that is applicable to both output scores and decisions (cf.~Fig.~\ref{fig:score-based-metric} (Right) for an overview).

We begin this section by first outlining desirable properties of predictive multiplicity metrics for probabilistic classifiers.
Motivated by the potential individual-level harm incurred by an arbitrary choice of model in the Rashomon set, we focus on per-sample multiplicity metrics. 
We formally define Rashomon Capacity in terms of the KL-divergence between the output scores of classifiers in the Rashomon set. We then use Rashomon Capacity to define a predictive multiplicity metric for individual samples in a dataset. 
We present these definitions in the ideal case where the population loss can be computed exactly, and discuss empirical approximation in Section~\ref{sec:challenges}.

\subsection{Properties of multiplicity metrics for probabilistic classifiers}
Consider a fixed data distribution $P_{X, Y}$ and a corresponding Rashomon set $\calR(\calH, \epsilon)$ for a classification problem with $c$ classes. For a given sample $\bx_i\in \calD$, we collect all possible output scores produced by models in $\calR(\calH, \epsilon)$, and define the \emph{$\epsilon$-multiplicity set}  as
\begin{equation}\label{eq:mult-set}
    \calM_\epsilon(\bx_i) \triangleq \left\{h(\bx_i); h\in \calR(\calH,\epsilon) \right\} \subseteq \Delta_c.
\end{equation}

Let $m(\cdot)$ be a measure of predictive multiplicity, and $m(\calM_\epsilon(\bx_i))$ be the predictive multiplicity of sample $\bx_i$. Which properties should $m(\cdot)$ have? 
Ideally, we expect $m(\calM_\epsilon(\bx_i))$ to be a bounded value in $[1, c]$, since at least one class is assigned to sample $\bx_i$, and at most $c$ different classes could be assigned to $\bx_i$. 
Moreover, if $m(\calM_\epsilon(\bx_i)) = 1$ (i.e., a predictive multiplicity of 1), then one would expect that only one score is produced for $\bx_i$ and, thus, all predictions in  $\calM_\epsilon(\bx_i)$ are  exactly the same. Similarly, if $m(\calM_\epsilon(\bx_i)) = c$ (i.e., predictive multiplicity equal to the number of classes),  then there must exist $c$ models $\{h_1, \cdots, h_c\}\subseteq \calR(\calH,\epsilon)$ such that $h_j(\bx_i)=\mathbf{e}_j$. In other words, each of the $c$ classes  can be assigned to the sample, yielding a predictive multiplicity of $c$.
Finally, $m(\calM_\epsilon(\bx_i))$ should be monotonic in  $\calM_\epsilon(\bx_i)$, i.e., if $\calM_\epsilon(\bx_i) \subseteq \calM'_\epsilon(\bx_i)$, then $m(\calM_\epsilon(\bx_i)) \leq m(\calM'_\epsilon(\bx_i))$. We summarize these desirable properties of predictive multiplicity metrics in the following definition.

\begin{defn}\label{def:metric}
Let $\calP_c$ be the power set\footnote{We exclude the empty set $\emptyset$ from $\calP_c$, since the multiplicity of an empty scores set is not well-defined.} of the probability simplex $\Delta_c$. The function $m:\calP_c\setminus\emptyset\to \mathbb{R}$ is a \emph{predictive multiplicity metric}\footnote{We use the term ``metric'' loosely, not in the sense of defining a metric space over a set.} if for any $\calA, \calB \in \calP_c$
\begin{enumerate}
    \item $1\leq m(\calA) \leq c$;
    \item $m(\calA) = 1$ if and only if $|\calA|\leq 1$;
    \item $m(\calA) = c$ if and only if $\mathbf{e}_k \in \calA$ for $k \in [c]$, i.e., $ \calA$ contains the corner points of $\Delta_c$;
    \item $m(\calA) \leq m(\calB)$ if  $\calA \subseteq \calB$.
\end{enumerate}
\end{defn}
We introduce next a predictive multiplicity metric called Rashomon Capacity that satisfies all properties above. In the rest of the paper, when the $\epsilon$-multiplicity set $\calM_\epsilon(\cdot)$ is clear from context, we use $m(\bx_i)$ as shorthand for $m(\calM_\epsilon(\bx_i))$.

\subsection{Rashomon Capacity}
Our goal is to quantify predictive multiplicity in terms of the score variations assigned to each sample $\bx_i$ in a dataset $\calD$, given a Rashomon set $\calR(\calH, \epsilon)$ and the corresponding $\epsilon$-multiplicity set $\calM_\epsilon(\bx_i)$. 
Note that an element in $\calM_\epsilon(\bx_i)$ is a probability distribution over $c$ classes. Thus, it is natural to adopt  divergence  measures for distributions to capture the ``variation'' of  scores in $\calM_\epsilon(\bx_i)$. 
From a geometric viewpoint, a larger spread in scores indicates a greater amount of predictive multiplicity for a given sample $\bx_i$. 

Assume a probability measure (or ``weight'') $P_M$ across models in $\calR(\calH, \epsilon)$ (and therefore each score in $\calM_\epsilon(\bx_i)$), where $M$ denotes the random variable of selecting/sampling the models in the Rashomon set. 
Intuitively, if $P_M$ assigns mass 1 to a single model and 0 to all other models in the Rashomon set, then the output of only one model is considered. 
Alternatively, if $P_M$ is the uniform distribution, then  the outputs of every model in the set are equally weighed. Given a divergence measure between distributions $d(\cdot \| \cdot)$, we quantify  the spread of the scores in $\calM_\epsilon(\bx_i)$ by
\begin{equation}\label{eq:score-spread}
    \rho(\calM_\epsilon(\bx_i), P_M) \triangleq \inf\limits_{\bq\in\Delta_c} \E_{h  \sim P_M} d(h(\bx_i)\|\bq).
\end{equation}
Here, the minimizing $\bq$ acts as a ``center of gravity'' or ``centroid'' for the outputs of the classifiers in the Rashomon set for a chosen distribution $P_M$ across models. Analogously, the quantity $ \rho(\calM_\epsilon(\bx_i), P_M)$ can be understood as a measure of ``spread'' or ``inertia'' across model outputs. We select the distribution $P_M$ that results in the largest spread in scores:
\begin{equation}\label{eq:rashomon-capacity}
    C_d(\calM_\epsilon(\bx_i)) \triangleq \sup\limits_{P_M}\rho(\calM_\epsilon(\bx_i), P_M) = \sup\limits_{P_M}\inf\limits_{\bq\in\Delta_c} \E_{h  \sim P_M} d(h(\bx_i)\|\bq).
\end{equation}

The missing element is the choice of divergence measure $d(\cdot\|\cdot)$. A natural candidate is cross-entropy (or log-loss)  $d(h(\bx_i)\|\bq)=-h(\bx_i)^\top \log \bq,$ since this is the standard loss used for training and evaluating probabilistic classifiers. Alas, the minimal
cross-entropy $\min_{\bq} -h(\bx_i)^\top\log \bq = h(\bx_i)^\top \log h(\bx_i)$ is not 0 and depends on $h(\bx_i)$. Consequently, if one were to choose $d(\cdot\|\cdot)$ to be cross-entropy, the minimum value of  \eqref{eq:rashomon-capacity} would not be consistent and would depend on $\bx_i$ --- even when the  outputs across all models in the Rashomon set match! Thus, we shift  cross-entropy by its minimum $-h(\bx_i)^\top\log h(\bx_i)$. This results in KL-divergence as our divergence measure of choice: $D_{KL}(h(\bx_i)\|\bq)=-h(\bx_i)^\top\log \bq + h(\bx_i)^\top\log h(\bx_i)$. 
Putting it all together, we next formally define the spread in scores measured using KL-divergence as \emph{Rashomon Capacity}.

\begin{defn}\label{def:rashomon-capacity}
Given a sample $\bx_i$, a Rashomon set $\calR(\calH, \epsilon)$, and the corresponding $\epsilon$-multiplicity set $\calM_\epsilon(\bx_i)$, the Rashomon Capacity\footnote{We consider logarithms in base 2, and the unit of $C(\calM_\epsilon(\bx_i))$ is bit. We include further discussions and a geometric interpretation of Rashomon Capacity in Appendices~\ref{appendix:other-metrics} and~\ref{appendix:geometric-interpretation} respectively. } is defined as 
\begin{equation}
\label{eq::RC}
    m_C(\bx_i)\triangleq 2^{C(\calM_\epsilon(\bx_i))},\;\text{where}\; C(\calM_\epsilon(\bx_i)) = \sup\limits_{P_M}\inf\limits_{\bq\in\Delta_c} \E_{h  \sim P_M} D_{KL}(h(\bx_i)\|\bq),
\end{equation}
where the supremum in the right-hand side  is taken over all probability measures $P_M$ over $\calR(\calH, \epsilon)$.
\end{defn}

The exponent $C(\calM_\epsilon(\bx_i))$ is ubiquitous in information theory; in fact, $C(\calM_\epsilon(\bx_i))$ is the \emph{channel capacity} \citep{cover1999elements} of a channel $P_{Y|M}$ whose rows are the entries of $\calM_\epsilon(\bx_i)$. 
This connection motivates the name ``Rashomon Capacity'' and is useful for proving that $m_C(\bx_i)$ is indeed a predictive multiplicity metric, as stated in the next proposition.
\begin{prop}\label{prop:rashomon-capacity}
The function $m_C(\cdot)=2^{C(\calM_\epsilon(\cdot))}: \calX \to [1, c]$ satisfies all properties of a predictive multiplicity metric in  Definition~\ref{def:metric}.
\end{prop}
\begin{proof}
See Appendix~\ref{appendix-proof:prop1}.
\end{proof}
In contrast, ambiguity and discrepancy in \eqref{eq:ambiguity-discrepancy} do not satisfy the properties of a predictive multiplicity metric outlined in Definition~\ref{def:metric}.
An interesting connection between ambiguity and Rashomon Capacity is that ambiguity measures the fraction of samples in a dataset with non-zero Rashomon Capacity. 
In addition, Rashomon Capacity is fundamentally different from the size of a Rashomon set, in the sense that a larger Rashomon set does not necessarily lead to a larger Rashomon Capacity.
Using a binary classification problem as an example, consider two Rashomon sets with scores
\begin{equation}
\begin{aligned}
    \calR_1(\calH, \epsilon) &= \{h_1, h_2, h_3\},\; h_1(\bx_i) = [0.45, 0.55], h_2(\bx_i) = [0.50, 0.50], h_3(\bx_i) = [0.60, 0.40],\\
    \calR_2(\calH, \epsilon) &= \{h_1, h_2\},\; h_1(\bx_i) = [0.85, 0.15], h_2(\bx_i) = [0.10, 0.90].
\end{aligned}
\end{equation}
$\calR_2(\calH, \epsilon)$ has a larger Rashomon Capacity than $\calR_1(\calH, \epsilon)$, 
 albeit $|\calR_2(\calH, \epsilon)| = 2 < |\calR_1(\calH, \epsilon)| = 3$.

\subsection{Rashomon Capacity in score and decision domains}
Rashomon Capacity is defined in terms of the raw outputs in $\Delta_c$ of a probabilistic classifier; therefore it can also be evaluated with decisions, since a decision, after one-hot encoding, still lies in the probability simplex (at a vertex). 
This allows Rashomon Capacity to provide a more nuanced view of score variation, and can be used to identify the number of ``conflicting'' classes in the predictions produced by models in the Rashomon set (i.e., Rashomon Capacity is between $1$ and $c$).

Taking a ternary classification as an example, for three score vectors $[0.49, 0.51,0]$ and $[0.51, 0.49,0]$, the Rashomon Capacity of the raw scores for these samples is close to 1. For the thresholded scores ($[0,1,0]$ and $[1,0,0]$) the Rashomon Capacity is now 2.
Since Rashomon Capacity is between 1 and 3, this indicates that the confusion is between 2 classes instead of 3. 
In contrast, prior metrics that also operate on thresholded scores, such as ambiguity and discrepancy \eqref{eq:ambiguity-discrepancy}, only capture agreement among predictions.
In this sense, score-level metrics could potentially provide a finer characterization of predictive multiplicity.
This \emph{does not} mean that multiplicity should only be reported in the score domain.
On the contrary, our suggestion is that multiplicity should be measured at both the score and threshold levels and reported to stakeholders, as scores and decisions paint different pictures on how models conflict.

\section{Computational challenges of multiplicity metrics}\label{sec:challenges}
The computation of any multiplicity metric requires an approximate characterization of the Rashomon set---even for simple hypothesis spaces such as logistic regression.
For instance, the computation of the Rashomon ratio involves estimating $\vol(\calR(\calH, \epsilon))$, which is a level set estimation problem \citep{mason2021nearly}, and is computationally infeasible\footnote{An exact computation of $\vol(\calR(\calH, \epsilon))$ in a special case, e.g., ridge regression where the Rashomon set forms an ellipsoid, is provided in \citet[Section~5.1]{semenova2019study}.} when the hypothesis space $\calH$ is large \citep{bachoc2021sample}.
For a logistic regression, the exact form of the pattern Rashomon ratio is not tractable due to the non-linearity of the maximum likelihood ratio \citep{hastie2009elements}.
Similarly, the computation of ambiguity/discrepancy requires solving an optimization over the Rashomon set and can be computationally burdensome since 0-1 loss is not differentiable.

Given a dataset and a hypothesis space, there are two core challenges in computing any predictive multiplicity metric (see Definition \ref{def:metric}).
The first challenge is selecting an appropriate Rashomon parameter $\epsilon$, since the smallest achievable test loss $L(h_\theta)$ is unknown and only empirically approximated using a dataset with finite samples. 
The second challenge is approximating the Rashomon set without exhaustively searching the hypothesis space\footnote{In neural networks, for example, the size of the Rashomon set is determined by the number of local minima, which grows exponentially many with the number of parameters \citep{auer1995exponentially}.}. Next, we discuss strategies for  addressing these two challenges. 

\paragraph{Selection of the Rashomon parameter $\epsilon$.} 
The value of $\epsilon$ can be set relative to the performance of a reference model. 
A natural choice of the reference model is the empirical risk minimizer, i.e., $h_{\theta^*}$, where $\theta^* \in \argmin_{\theta\in\Theta} \hat{L}(h_\theta)$ (see Section \ref{sec:background} for notation).
Here, we can set  $\epsilon = \hat{L}(h_{\theta^*}) + \epsilon'$ with $\epsilon' \geq 0$, i.e., the Rashomon parameter depends on the minimum empirical loss. 
The parameter $\epsilon'$, in turn, can be selected in terms of the upper boundary of a confidence interval around the empirical minimum. For example, the confidence interval can be estimated via  \emph{bootstrapping} (see Appendix~\ref{appendix:wo-neural-networks} for bootstrapped loss intervals). 
Naturally, $\epsilon'$ depends on the size of the test set used for evaluating $h_{\hat{\theta}}$---when the dataset has $n$ samples, usually $\epsilon'$ will be of the order $O(1/\sqrt{n})$ \citep{shalev2014understanding}.
If more samples are available to evaluate model performance (rendering narrower confidence intervals), this parameter's value will decrease.

The discussion above motivates the one-sided definition for the Rashomon set in \eqref{eq:rashomon-set}. 
If one were able to find the unique classifier that outperforms all others in terms of population loss, then it is justified to use this classifier in practice. 
In fact, this point is eloquently made in \citet{creel2021algorithmic}, which argues  that there is no (moral) harm in using the most accurate classifier, i.e., selecting a classifier that has a provable smaller average loss (or higher accuracy) than another.  
Existing works on multiplicity, e.g., \citet{semenova2019study} and \citet{marx2020predictive}, also adopt a one-sided Rashomon set definition.

\paragraph{The Rashomon subset.}
With limited computational power and memory, exploring the Rashomon set, i.e., searching the hypothesis space $\calH$ to find all models with test losses smaller than $\epsilon$, is challenging.
For several hypothesis spaces (e.g., neural networks), we are only able to acquire a small number of models in the Rashomon set. 
We denote the collection of these models as a \emph{Rashomon subset} $\widetilde{\calR}(\calH, \epsilon)$.

The Rashomon subset can be used to approximate the ``true'' Rashomon set $\calR(\calH, \epsilon)$ in the evaluation of multiplicity metrics.
We can construct a Rashomon subset with $K$ models by choosing the Rashomon parameter $\epsilon$ in terms of a reference model $h_{\theta^*}$, i.e.,
\begin{equation}\label{eq:rashomon-subset}
    \widetilde{\calR}(\calH, \epsilon') \triangleq \{h_{\theta_i}\in\calH; L(h_{\theta_i}) \leq \hat{L}(h_{\theta^*}) + \epsilon'\}_{i=1}^K \subseteq \calR(\calH, \hat{L}(h_{\theta^*}) + \epsilon').
\end{equation}
For example, the Rashomon ratio can be approximated as $K/\vol(\calH)$.
Similarly, we can optimize for ambiguity/discrepancy in \eqref{eq:ambiguity-discrepancy} over $h_{\theta_i}\in\widetilde{\calR}(\calH, \epsilon')$.
Finally, we can approximate the $\epsilon'$-multiplicity subset $\widetilde{\calM}_{\epsilon'}(\bx_i) \triangleq \left\{h_{\theta_i}(\bx_i); h\in \widetilde{\calR}(\calH,\epsilon') \right\}$ in \eqref{eq:mult-set} for Rashomon Capacity.

In this sense, evaluating multiplicity metrics boils down to finding a Rashomon subset $\widetilde{\calR}(\calH, \epsilon')$ and computing the metric for that set. 
This is an approximation of the true multiplicity, yet becomes more accurate with an increasing $K$. 
Next, we provide an algorithm based on model weight perturbation to find a Rashomon subset for estimating Rashomon Capacity for any differentiable model.

Note that, unless the Rashomon set can be fully characterized, all estimates of multiplicity metrics are underestimates based on approximating the true Rashomon set by a Rashomon subset.
We emphasize that identifying and disclosing predictive multiplicity --- even if an underestimate --- is still critical for models deployed in applications of individual-level consequence (e.g., healthcare, education, lending), and is better than the current practice reporting no multiplicity at all. 

\paragraph{Computing Rashomon Capacity.}
The definition of Rashomon Capacity in \eqref{eq::RC} does not assume a finite cardinality of the Rashomon set. 
Remarkably, even when the Rashomon set has infinite cardinality, the value of Rashomon Capacity of a sample can be recovered by considering only a small number of models in the Rashomon set.
In fact, for each sample $\bx_i$, there exists a $\epsilon$-multiplicity subset of at most $c$ models that fully captures the variation in scores. 
This statement is formalized by the next proposition, which can be proven by applying Carath\'eodory's theorem \citep{caratheodory1911variabilitatsbereich}.
\begin{prop}\label{prop:sampling-rashomon-set}
For each sample $\bx_i\in \calD$, there exists a $\epsilon$-multiplicity subset  $\widetilde{\calM}_\epsilon(\bx_i)\subseteq\calM_\epsilon(\bx_i)$ with $|\widetilde{\calM}_\epsilon(\bx_i)|\leq c$  that fully captures the spread in scores for $\bx_i$ across the Rashomon set, i.e.,  $C(\widetilde{\calM}_\epsilon(\bx_i)) = C(\calM_\epsilon(\bx_i))$. 
In particular, there are at most $c$ models in a Rashomon subset $\widetilde{\calR}(\calH,\epsilon)$ whose output scores yield the same Rashomon Capacity for $\bx_i$ as the entire Rashomon set.
\end{prop}
\begin{proof}
See Appendix~\ref{appendix-proof:prop2}.
\end{proof}
Proposition~\ref{prop:sampling-rashomon-set}  implies that, for each sample, there exists a Rashomon subset with $c$ models that captures all score variation. 
In other words, the value of Rashomon Capacity remains the same regardless if we measure multiplicity on this subset or on the entire Rashomon set.
This allows us to circumvent the task of characterizing the entire Rashomon set (which has potentially infinite models), and focus on identifying $c$ models per sample that maximize score variations while still satisfying a target loss constraint.
We describe next a  method (described in detail in Algorithm~\ref{alg:awp-RS}) based on weight perturbation that obtains $c$ models in the Rashomon subset for each sample.

Given a sample $\bx_i$, we obtain models with output predictions $\bp_k$
by approximately solving the following optimization problem which maximizes the output score for class $k$: 
\begin{equation}\label{eq:search-parameter}
    \bp_k = h_{\hat{\theta}}(\bx_i),\;\text{where}\; \hat{\theta} = \argmax \limits_{h_\theta\in\calR(\calH, \epsilon)} [h_\theta(\bx_i)]_k,\;\forall k = 1, 2, \cdots, c.
\end{equation}
To solve \eqref{eq:search-parameter}, for each $k$, we set the objective to be $\min_{\theta\in\Theta}-[h_{\theta}(\bx_i)]_k$, compute the gradients, and update the parameter $\theta$ until $L(h_\theta) > \epsilon$.
Given a pre-trained model in the Rashomon set, \eqref{eq:search-parameter} can be viewed as an adversarial weight perturbations (\texttt{AWP}) technique to explore the Rashomon set \citep{wu2020adversarial, tsai2021formalizing} (see Appendix~\ref{appendix:awp-logistic} for exact weight perturbation on logistic regression). 

With the discrete $\epsilon$-multiplicity subset obtained by solving \eqref{eq:search-parameter}, Rashomon Capacity can be computed by standard procedures such as the Blahut–Arimoto (BA) algorithm \citep{blahut1972computation, arimoto1972algorithm}. 
The BA algorithm is a class of iterative algorithms for numerically computing \emph{discrete} channel capacity (or more generally, the rate-distortion function), see Appendix~\ref{appendix:ba} for more details.
Note that AWP may still underestimate the true Rashomon Capacity, yet it greatly improves the estimates compared to prior work and is less computationally intensive than sampling, as shown in the next section.

The procedure in \eqref{eq:search-parameter} reveals a desirable property of a Rashomon subset: it should include models with significant score variations. 
Similarly, a desirable Rashomon subset for accurately evaluating ambiguity/discrepancy is one with models that have most score disagreement.
This property also explains why collecting a Rashomon subset by straightforwardly sampling models \citep{semenova2019study} in the Rashomon set could be inefficient (cf. Algorithm~\ref{alg:sampling-RS}) since the randomly sampled models would not necessarily have a significant score disagreement/variation.
In fact, the sampling strategy could require a significant amount of models; see for example in \citet[Section~6]{semenova2019study}, a Rashomon subset consisting of 250k decision trees.

\section{Empirical study}\label{sec:exp}
We illustrate how to measure, report, and potentially resolve predictive multiplicity of probabilistic classifiers using Rashomon Capacity on UCI Adult \citep{Lichman:2013}, COMPAS \citep{angwin2016machine}, HSLS \citep{ingels2011high}, and CIFAR-10 datasets \citep{krizhevsky2009learning}.
UCI Adult and COMPAS are two binary classification datasets on income and recidivism prediction, respectively, and are widely used in fairness research \citep{mehrabi2021survey}.
The HSLS is an education dataset, collected from high school students in the USA, whose features include student and parent information (see Appendix~\ref{appendix:data-preprocessing} for details).
We created a binary label $Y$ from students' $9^\text{th}$-grade math test scores (i.e., top 50\% vs. bottom 50\%).
We select the first three datasets  to illustrate the effect of predictive multiplicity on individuals.
Finally, we include the CIFAR-10 dataset to demonstrate how to report Rashomon Capacity in multi-class classification.

For the classifiers, we adopt feed-forward neural networks for the first three datasets, and a convolutional neural network VGG16 \citep{simonyan2014very} for CIFAR-10.
For more information on the datasets, neural network architectures, and training details, see Appendix~\ref{appendix:exp-setup}.
All numbers reported are evaluated on the test set.
 
\begin{figure}[t!]
     \centering
     \begin{subfigure}[b]{0.24\textwidth}
         \centering
         \includegraphics[width=\textwidth]{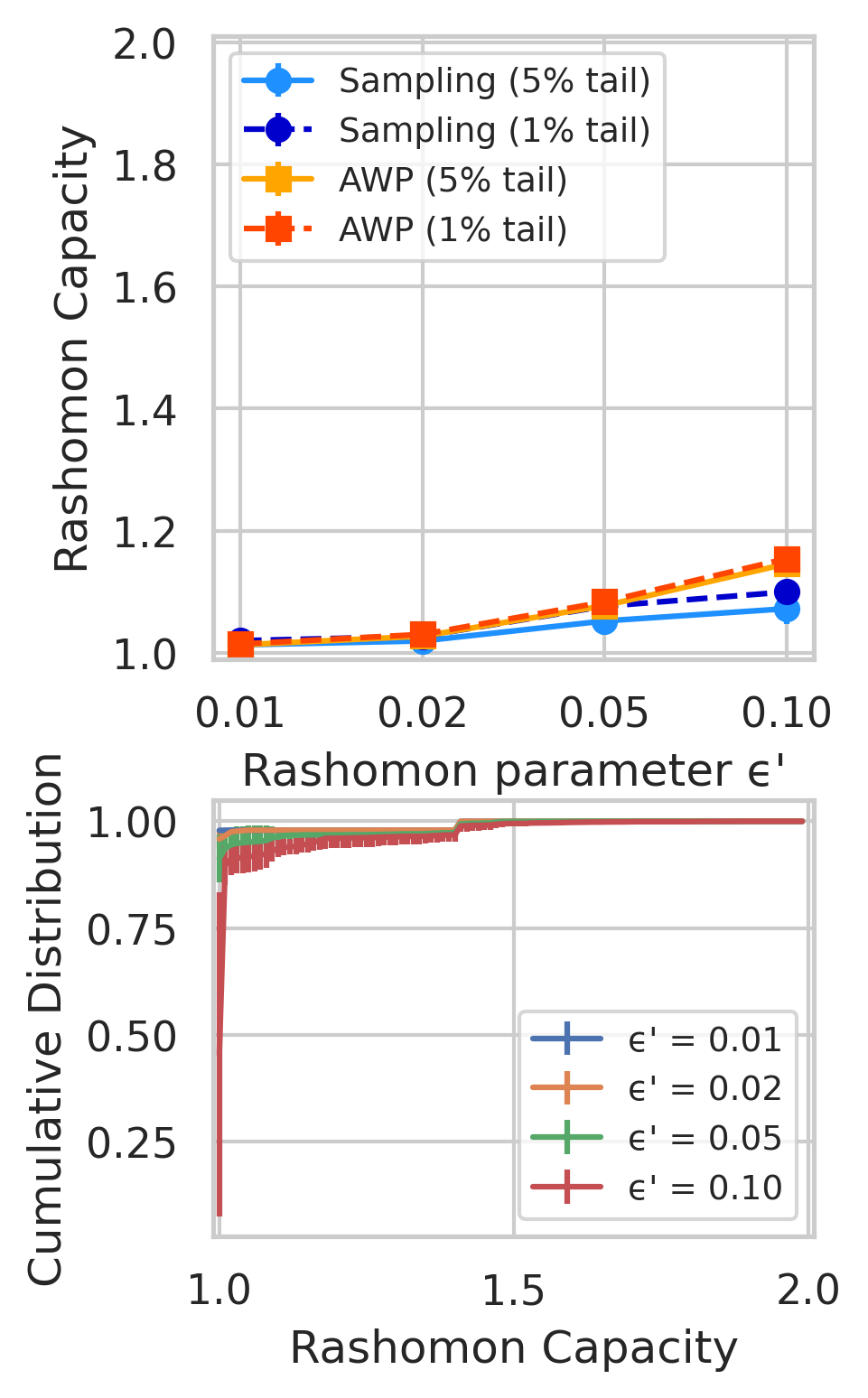}
         \caption{UCI Adult (80.28\%)}
     \end{subfigure}
     \begin{subfigure}[b]{0.24\textwidth}
         \centering
         \includegraphics[width=\textwidth]{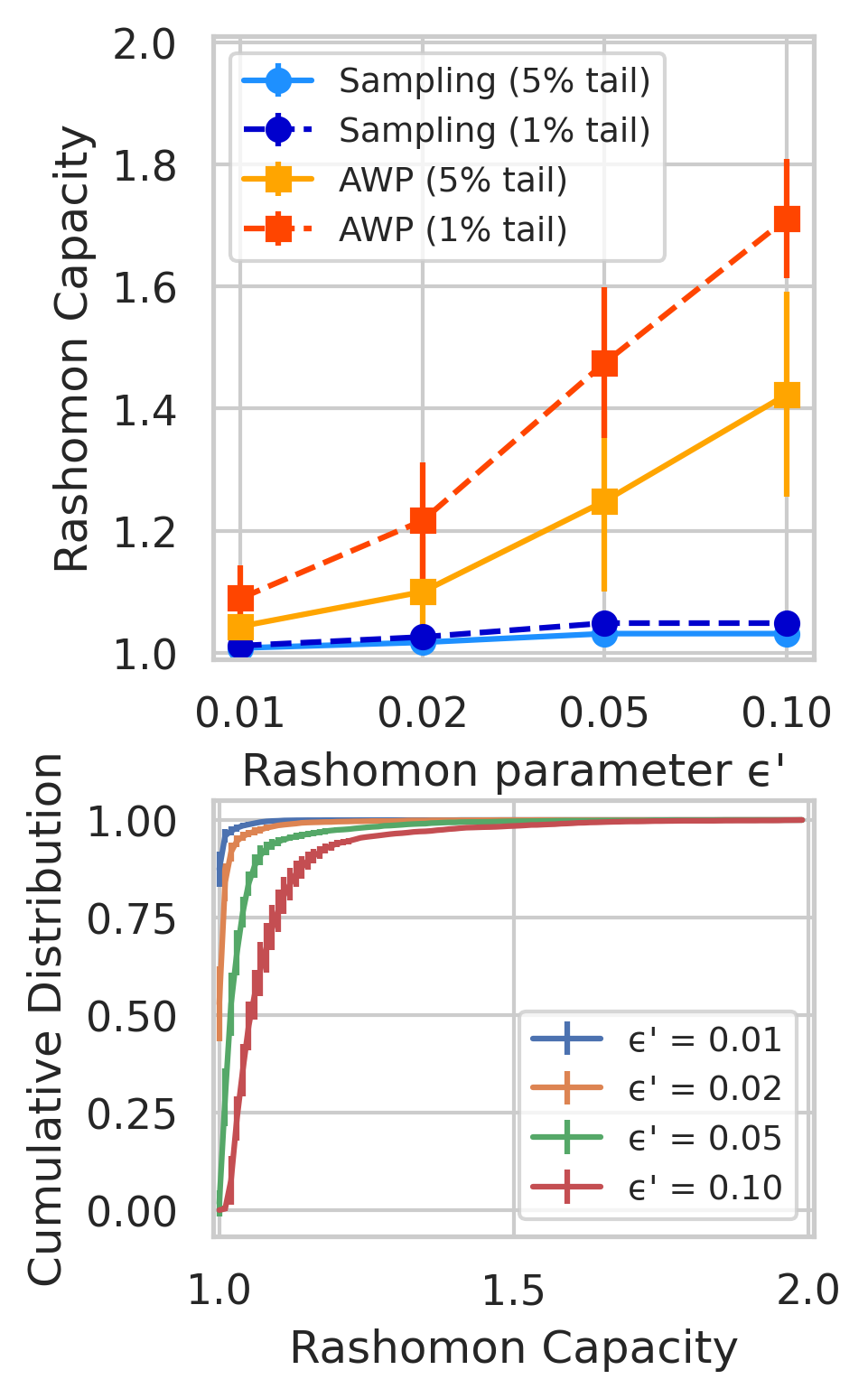}
         \caption{COMPAS (66.22\%)}
     \end{subfigure}
     \begin{subfigure}[b]{0.24\textwidth}
         \centering
         \includegraphics[width=\textwidth]{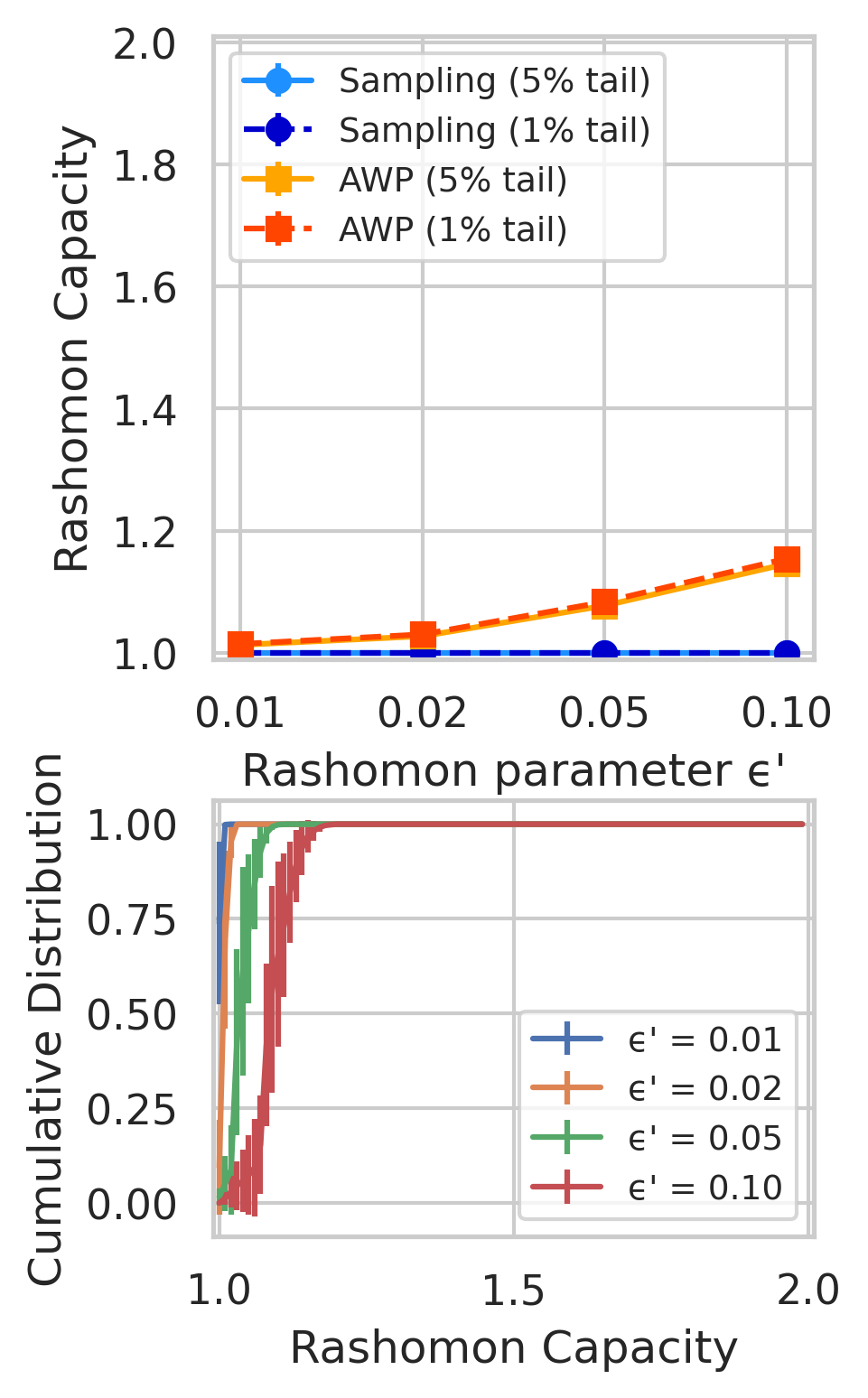}
         \caption{HSLS (70.39\%)}
     \end{subfigure}
     \begin{subfigure}[b]{0.254\textwidth}
         \centering
         \includegraphics[width=\textwidth]{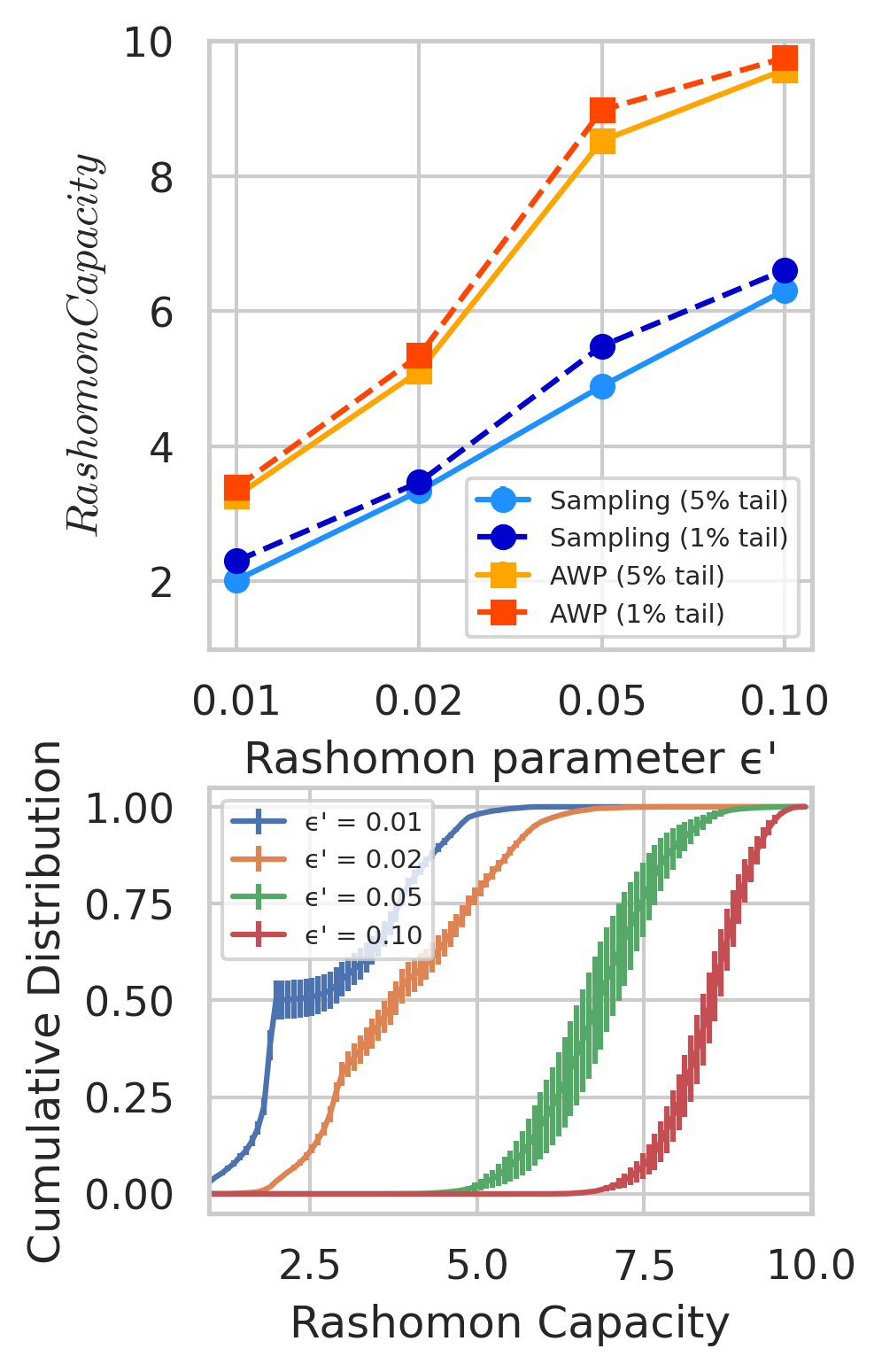}
         \caption{CIFAR-10 (81.67\%)}
     \end{subfigure}
    \caption{For each dataset (percentage is test accuracy), the top figure shows the mean and standard error of the largest 1\% and 5\% (1\% tail and 5\% tail in the legend) Rashomon Capacity among all the samples with difference Rashomon parameter $\epsilon$. Two methods are used to obtain models from the Rashomon set, \texttt{AWP} \eqref{eq:search-parameter} and random \texttt{sampling}. The bottom figure shows the cumulative distribution of the Rashomon Capacity of all the samples obtained by \texttt{AWP}. Each point is generated with 5 repeated splits of the dataset. }
    \label{fig:capacity-profiles}
\end{figure}

\paragraph{Measuring and reporting predictive multiplicity via Rashomon Capacity.}
We evaluate two methods, \texttt{sampling} with different weight initialization seeds \citep{semenova2019study} and \texttt{AWP} \eqref{eq:search-parameter}, and report the Rashomon Capacity in Fig.~\ref{fig:capacity-profiles}.
For \texttt{sampling}, we construct a Rashomon subset $\widetilde{\calR}(\calH, \epsilon')$ with 100 models by different random initialization (cf.~\eqref{eq:rashomon-subset} with $K=100$). 
For \texttt{AWP}, the $\epsilon'$-multiplicity subsets $\widetilde{\calM}_{\epsilon'}(\bx_i)$ of each sample $\bx_i$ is constructed by the $\bp_1, \cdots, \bp_c$ obtained from \eqref{eq:search-parameter}. 
For example, when $\epsilon' = 0.01$, the models in the Rashomon set $\widetilde{\calR}(\mathcal{H}, 0.01)$ achieves small and statistically indistinguishable test losses from each other; however, the Rashomon Capacity is non-zero for a significant fraction of samples.
In particular, we showcase the average top 1\% and top 5\% of the Rashomon Capacity from all samples for different $\epsilon'$, and the cumulative distribution of the Rashomon Capacity across the samples. 
As the Rashomon parameter increases, both \texttt{sampling} and \texttt{AWP} lead to higher Rashomon Capacity since the Rashomon set gets larger.
Observe that the increase of Rashomon Capacity when $\epsilon'$ varies from $0.01$ to $0.02$ is different for UCI Adult and COMPAS datasets, since the choice of $\epsilon'$ is data-dependent.
The \texttt{AWP} \eqref{eq:search-parameter} achieves higher Rashomon Capacity than random \texttt{sampling} as \texttt{AWP} intentionally explores the Rashomon set that maximizes the scores variations. It is important to keep in perspective that \emph{each sample} in the high-Rashomon Capacity tail displayed in Fig.~\ref{fig:capacity-profiles} corresponds to an \emph{individual} who receives conflicting predictions. In applications such as criminal justice and education, conflicting predictions for even one individual should be reported in, e.g., model cards \citep{mitchell2019model}. 

For the experiments in Figure~\ref{fig:capacity-profiles}, the estimated Rashomon ratio are almost all zeros since the hypothesis space, parameterized by millions of parameters, is always significantly larger.
In SM, we further report (i) other strategies to explore the Rashomon set in Appendix~\ref{appendix:flipping-fgsm}, (ii) Rashomon Capacity evaluated with decisions instead of scores in Appendix~\ref{appendix:rc-decision}, and (iii) other metrics of multiplicity such as ambiguity/discrepancy in Appendix~\ref{appendix:ambiguity-discrepancy}.

\paragraph{Resolving predictive multiplicity by greedy model selection.} 
Predictive multiplicity could be resolved by, for example, selecting a subset of models (potentially of size one in the ideal case) and releasing the scores to a stakeholder.
We propose a greedy model selection procedure to select a subset of competing classifiers for resolving predictive multiplicity. 
Given $R$ competing classifiers, the goal is to select $r$ models ($r < R$) that result in distributions of the Rashomon Capacity similar to that of the original $R$ models. 
Starting from a dataset $\calD$ and a Rashomon subset $\widetilde{\calR}(\calH, \epsilon')$, this can be implemented by (i) initializing a set $\calA$ of models by randomly selecting a model in $\widetilde{\calR}(\calH, \epsilon')$, (ii) growing  $\calA$ by adding one model from $\widetilde{\calR}(\calH, \epsilon')$  that maximizes the average Rashomon Capacity across $\calD$, and (iii) stopping until there are $r$ models in $\calA$.
This greedy model selection is inspired by Property 4 (monotonicity) in Definition~\ref{def:metric}, since including the models to the set $\calA$ does not reduce capacity.
In Fig.~\ref{fig:resolve}, we models from the Rashomon sets for UCI Adult, COMPAS, HSLS and CIFAR-10 datasets respectively. Here, the hypothesis space are feed-forward neural networks (see details in Appendix~\ref{appendix:exp-setup}). Observe that only a small subset of the sampled models, selected by the greedy model selection procedure, is required to recover the distribution of the Rashomon Capacity. 
On COMPAS dataset, the 10 models obtained by the greedy model selection procedure capture the Rashomon Capacity computed with the original 163 models, i.e., these 10 models display most of the score variations. 

In Appendix~\ref{appendix:compas-case-study}, we provide a worked-out example with the COMPAS dataset on how Rashomon Capacity can be used to identify high-multiplicity samples.
We observe patterns in sex and prior convictions leading to samples having high Rashomon Capacity.
We further discuss promising strategies, e.g., ensemble methods, model calibration and weight regularization to resolve multiplicity, in Appendices~\ref{appendix:ensemble-calibration} and~\ref{appendix:wo-neural-networks}. 
We observe that the ensemble method could lead to a smaller Rashomon Capacity, and is a viable strategy for resolving multiplicity in small models, but may be infeasible for large, computationally expensive models.
In addition, weight regularization for logistic regression (e.g., LASSO or ridge penalties) could also reduce Rashomon Capacity. 
On the other hand, a perfectly calibrated classifier does not necessarily resolve multiplicity---a classifier whose predicted classes matches the true classes ``on average'' across samples does not necessarily translate to a consistent set of predictions for a single target sample across equally calibrated classifiers.

\begin{figure}[t!]
     \centering
     \begin{subfigure}[b]{0.23\textwidth}
         \centering
         \includegraphics[width=\textwidth]{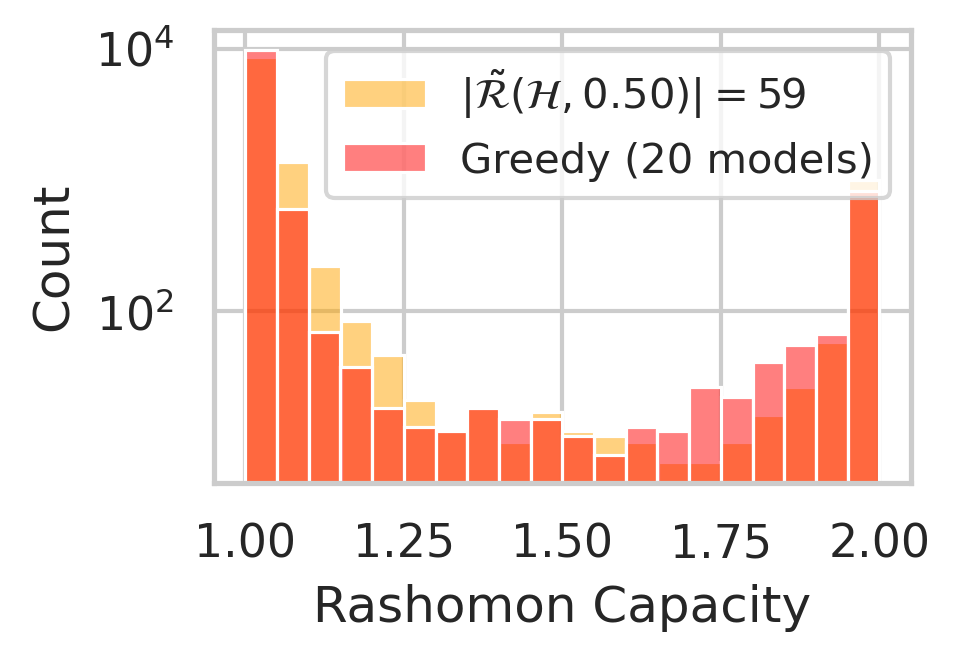}
         \caption{UCI Adult (83.47\%)}
     \end{subfigure}
     \begin{subfigure}[b]{0.23\textwidth}
         \centering
         \includegraphics[width=\textwidth]{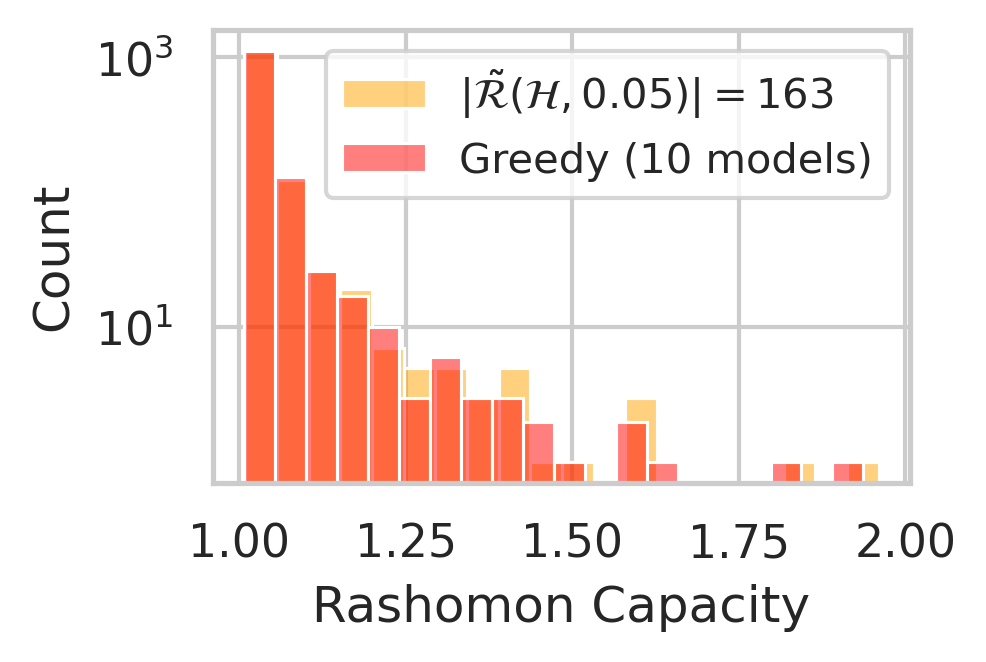}
         \caption{COMPAS (67.35\%)}
     \end{subfigure}
     \begin{subfigure}[b]{0.23\textwidth}
         \centering
         \includegraphics[width=\textwidth]{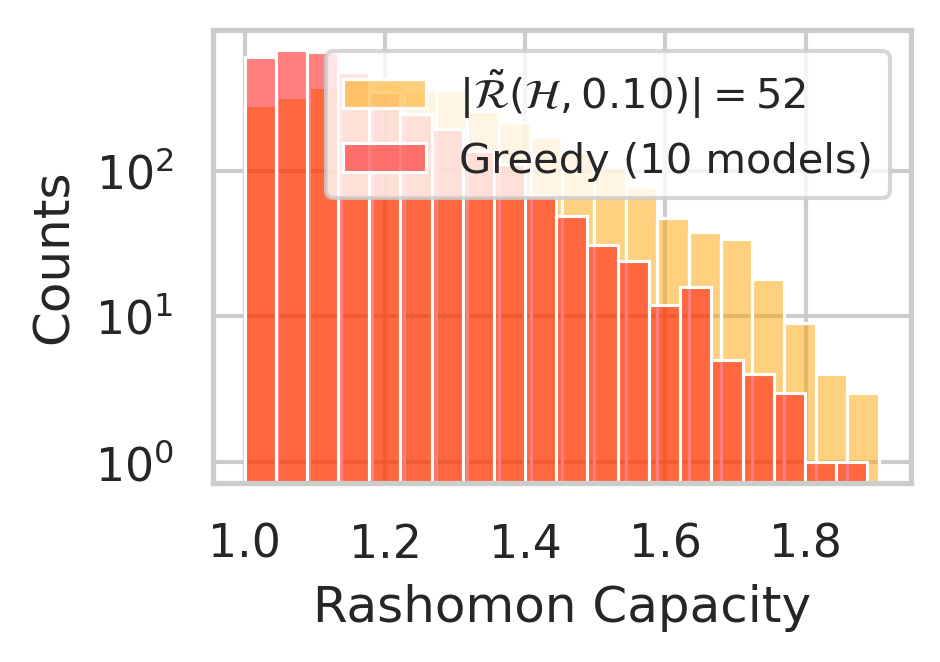}
         \caption{HSLS (69.91\%)}
     \end{subfigure}
     \begin{subfigure}[b]{0.23\textwidth}
         \centering
         \includegraphics[width=\textwidth]{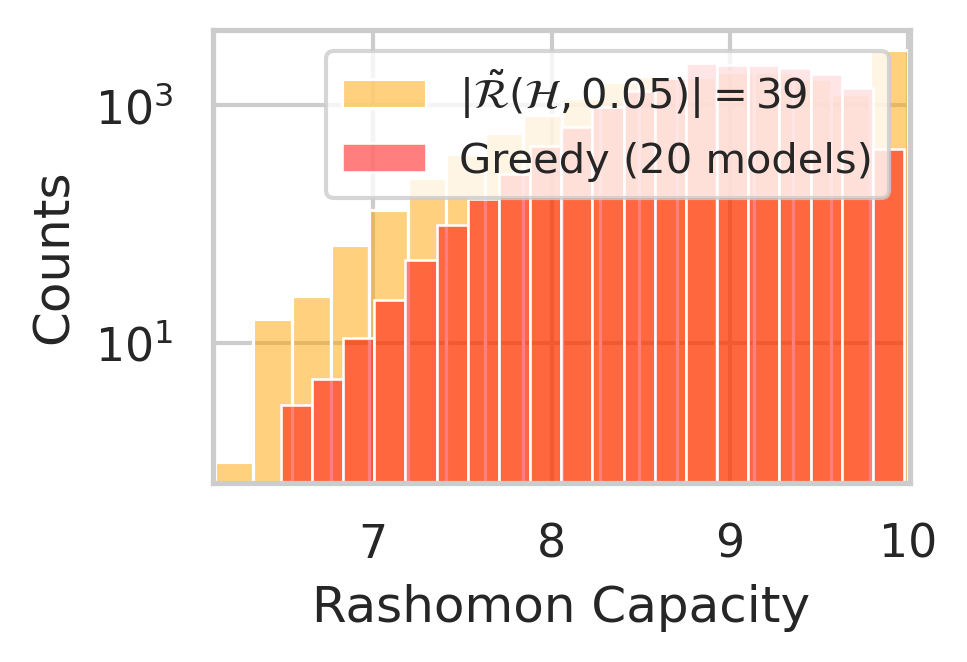}
         \caption{CIFAR-10 (80.33\%)}
     \end{subfigure}
    \caption{The distributions of the Rashomon Capacity for UCI Adult, COMPAS, HSLS and CIFAR-10 datasets ($\epsilon$ is percentage of mean test accuracy) obtained by sampling models from the Rashomon subset and applying greedy model selection procedure (Greedy in the legend) on the sampled models.}
    \label{fig:resolve}
\end{figure}

\section{Final remarks}\label{sec:final-remarks}
\paragraph{Limitations.}
The AWP \eqref{eq:search-parameter}, despite being more efficient than random sampling, still requires re-training/perturbing a significant amount of models, and is computationally burdensome when scaling up to large datasets with millions of samples. 
Rashomon Capacity in certain cases may seem small for already significant score variations across classes, due to the convexity of KL-divergence.
Indeed, this issue will occur for any strictly convex measure of divergence in \eqref{eq:score-spread}. 
We provide a further discussion of how to interpret the numerical values of Rashomon Capacity in Appendix~\ref{appendix:convexity}.

\paragraph{Future directions.}
First, overcoming the computational bottleneck to efficiently explore the Rashomon set is an impactful direction for optimization techniques.
Second, Rashomon Capacity could be generalized to other probability divergences, e.g., $f$-divergences \citep{csiszar1995generalized}, R\'enyi divergence \citep{renyi1961measures}, or Wasserstein distance \citep{vaserstein1969markov}.
This generalization could potentially provide further operational significance and tunability for measuring multiplicity, as long as the conditions in Definition~\ref{def:metric} are satisfied.
Third, ensemble methods could be a promising strategy to reduce predictive multiplicity that worth studying.

\paragraph{Broader impacts.}
The Rashomon effect impacts model selection \citep{rudin2019stop, hancox2020robustness, d2020underspecification, black2022model}, explainability \citep{pawelczyk2020counterfactual}, and fairness \citep{coston2021characterizing}.
\citet{rudin2019stop} suggested that, given the choice of competing models, machine learning practitioners should select interpretable models \emph{a priori}, rather than selecting a black-box model with conjectural explanations post-training.
\citet{hancox2020robustness} and \citet{d2020underspecification} argued that epistemic patterns, e.g., causality, should be reflected when selecting models the Rashomon set.
\citet{black2022model} further studied multiplicity in the context of the conventional bias-variance trade-off analysis.
Competing models in the Rashomon set may not only render conflicting explanations for predictions \citep{pawelczyk2020counterfactual} and measures of feature importance \citep{fisher2019all}, but also have inconsistent performance across population sub-groups. 
Consequently, the arbitrary  choice of a single model may result in unnecessary and discriminatory bias against vulnerable population groups \citep{coston2021characterizing}. 
See Appendix~\ref{appendix:explanation} for further discussion, including connections with individual fairness \citep{dwork2018individual}.

\clearpage
\bibliography{reference.bib}
\bibliographystyle{apalike}

\newpage
\appendix
\def\thesection{A.\arabic{section}}
\def\theequation{A.\arabic{equation}}
\def\thefigure{A.\arabic{figure}}
\def\thetable{A.\arabic{table}}
\def\thealgorithm{A.\arabic{algorithm}}
\setcounter{section}{0}
\setcounter{equation}{0}
\setcounter{figure}{0}  
\setcounter{table}{0}
\setcounter{algorithm}{0}

\section*{Appendix}
This appendix include omitted proofs for Proposition~\ref{prop:rashomon-capacity} and~\ref{prop:sampling-rashomon-set}, additional explanations and discussions, details on experiment setups and training, and additional experiments.
For clarity, the numbers with a prefix \emph{A.} refer to equations, figures, and tables in the appendix; numbers without the prefix refer to equations, figures, and tables in the main paper. 

\section{Omitted proofs}\label{appendix:proofs}
\subsection{Proof of Proposition~\ref{prop:rashomon-capacity}}\label{appendix-proof:prop1}
\begin{prop}
The function $m_C(\cdot)=2^{C(\calM_\epsilon(\cdot))}: \calX \to [1, c]$ satisfies all properties of a predictive multiplicity metric in  Definition~\ref{def:metric}.
\end{prop}
\begin{proof}
For clarity, we assume $|\calM_\epsilon(\bx_i)| = m$.
By the information inequality \citep[Theorem~2.6.3]{cover1999elements} the mutual information $I(M; Y)$ between the random variables $M$ and $Y$ (defined in Section~\ref{sec:rashomon-capacity}) is non-negative, i.e., $I(M; Y) \geq 0$.
Moreover, since $I(M; Y) = H(Y) - H(Y|M) \leq \log c$, we have $0 \leq C(\calM_\epsilon(\bx_i)) \leq \log c$, and therefore $1 \leq 2^{C(\calM_\epsilon(\bx_i))} \leq c$ (since we pick the log base to be 2).
If all rows in the transition matrix are identical, $Y$ is independent of $M$, and $H(Y|M) = \sum_{j=1}^m p_m(j) H(Y|M=j) = \sum_{j=1}^m p_m(j) H(Y) = H(Y)$; thus, $C(\calM_\epsilon(\bx_i)) \leq H(Y) - H(Y|M) = 0$ and $m_C(\bx_i) = 2^{C(\calM_\epsilon(\bx_i))} = 1$.
On the other hand, we denote the $c$ models in $\calR(\calH, \epsilon)$ which output scores are the ``vertices'' of $\Delta_c$ to be $m_1, \cdots, m_c$, then $H(Y|M=m_k) = 0,\; \forall k\in[c]$. $H(Y|M)$ is minimized to 0 by setting the weights $p_m$ on those c models to be $\frac{1}{c}$ and the rest to be $0$. 
Thus, $C(\calM_\epsilon(\bx_i)) = H(Y) - H(Y|M) = H(Y)$ and $m_C(\bx_i) = 2^{C(\calM_\epsilon(\bx_i))} = c$.
Finally, let $\calM_\epsilon^1(\bx_i) \subseteq \calM_\epsilon^2(\bx_i)$ with random variables $M_1$ and $M_2$ respectively.
Without loss of generality, assume that $\calM_\epsilon^1(\bx_i) = \{h_1(\bx_i), \cdots, h_r(\bx_i)\}$ and $\calM_\epsilon^2(\bx_i) = \calM_\epsilon^1(\bx_i) \cup \{h_{r+1}(\bx_i)\}$, and we have
\begin{equation}
\begin{aligned}
    \E_{h\sim P_{M_2}} D(h(\bx_i)\|\bq) &= \sum_{i=1}^{r+1} P_{M_2}(h_i) D_{KL}(h_i(\bx_i)\|\bq)\\
    &= \sum_{i=1}^{r} P_{M_2}(h_i) D_{KL}(h_i(\bx_i)\|\bq) + P_{M_2}(h_{r+1}) D_{KL}(h_{r+1}(\bx_i)\|\bq)\\
    &\geq \sum_{i=1}^{r} P_{M_2}(h_i) D_{KL}(h_i(\bx_i)\|\bq) = \E_{h\sim P_{M_1}} D(h(\bx_i)\|\bq).
\end{aligned}
\end{equation}
Therefore, $I(M_1; Y) = \inf_{P_{M_1}} \E_{h\sim P_{M_1}} D(h(\bx_i)\|\bq) \leq \inf_{P_{M_2}} \E_{h\sim P_{M_2}} D(h(\bx_i)\|\bq) = I(M_2; Y)$, and
\begin{equation}
\begin{aligned}
    I(M_1; Y) \leq I(M_2; Y) &\Rightarrow \sup\limits_{M_1} I(M_1; Y) \leq \sup\limits_{M_2} I(M_2; Y)\\
    &\Rightarrow C(\calM_\epsilon^1(\bx_i)) \leq C(\calM_\epsilon^2(\bx_i))\\
    &\Rightarrow m_C(\calM_\epsilon^1(\bx_i)) \leq m_C(\calM_\epsilon^2(\bx_i)),
\end{aligned}
\end{equation}
since the power function is monotonic.

We now prove the converse statements. Assume $m(\bx_i)=c$ and, thus, $C(\calM_\epsilon(\bx_i))=\log c$. Let $P_M$ be the capacity-achieving distribution over models. Then $I(M;Y)=\log c$ and, from non-negativity of entropy and the fact that the uniform distribution maximizes entropy, $H(Y)=c$ and $H(Y|M)=0$. Consequently, again from non-negativity of entropy, $H(Y|M=m)=0$ for all $m\in \mathsf{supp}(P_M),$ and thus $P_{Y|M=m}$ is an indicator function (i.e., given $M$, $Y$ is constant w.p.1). Since $H(Y)=c,$ the result follows.

Finally, let $C(\calM_\epsilon(\bx_i))=0$ and $P_M$ be the capacity-achieving input distribution. Then $Y$ and $M$ are independent and, thus, $P_{Y|M=m}=P_Y$ (i.e., all scores are identical) for all values $m\in \mathsf{supp}(P_M),$. Since this holds for the capacity-achieving $P_M$, which in turn is the maximimum across input distributions, the converse result follows.
\end{proof}

\subsection{Proof of Proposition~\ref{prop:sampling-rashomon-set}}\label{appendix-proof:prop2}
\begin{prop}
For each sample $\bx_i\in \calD$, there exists a subset  $\mathcal{A}\subseteq\calM_\epsilon(\bx_i)$ with $|\mathcal{A}|\leq c$  that fully captures the spread in scores for $\bx_i$ across the Rashomon set, i.e.,  $m_C(\bx_i) = 2^{C(\mathcal{A})}$. In particular, there are at most $c$ models in $\mathcal{R}(\calH,\epsilon)$ whose output scores yield the same Rashomon Capacity for $\bx_i$ as the entire Rashomon set.
\end{prop}
\begin{proof}
Carath\'eodory's theorem \citep{caratheodory1911variabilitatsbereich} states that if a point $x$ of $\R^d$ lies in the convex hull of a set $\calX$, then x can be written as the convex combination of at most $d+1$ points in $\calX$. 
Namely, there is a subset $\calX'$ of $\calX$ consisting of $d+1$ or fewer points such that $x$ lies in the convex hull of $\calX'$.

In our case, we consider the random variable $M$ of the Rashomon set $\calR(\calH, \epsilon)$ and $Y = \Delta_c \triangleq \{\bg \in \R^c; \sum_{i=k}^c [\bg]_k = 1,\forall k\; [\bg]_k \geq 0 \}$ is a $(c-1)$ dimensional space\footnote{In the main paper, we say $\Delta_c$ is a c-dimensional probability simplex, but it does not mean $\Delta_c$ is a c dimensional space. In fact, $\Delta_c$ is a $(c-1)$ dimensional space.}. 
We assume $|\calR(\calH, \epsilon)| = m$ in the following proof, but $\calR(\calH, \epsilon)$ could contain arbitrary large (or infinite) amount of output scores from the models in the Rashomon set.
There are also m output scores $\{h_1(\bx_i), \cdots, h_m(\bx_i)\}\in\Delta_c$ in the $\epsilon$-multiplicity set $\calM_\epsilon(\bx_i)$ for each sample $\bx_i \in \calD$.
By Carath\'eodory's theorem, since $\Delta_c$ is $(c-1)$ dimensional and is convex, any score $h(\bx_i)$ can be expressed by the convex combination of $(c-1)+1 = c$ scores. 
Moreover, let the $c$ scores be $\{h_1(\bx_i), \cdots, h_c(\bx_i)\}$, since Rashomon Capacity measures the spread of the scores, adding any score $h(\bx_i) \in \convexhull(h_1(\bx_i), \cdots, h_c(\bx_i))$ to the channel constructed by $\{h_1(\bx_i), \cdots, h_c(\bx_i)\}$ would not affect Rashomon Capacity.
\end{proof}

\section{Additional details}
\subsection{Predictive multiplicity: fairness, reproducibility, and security}\label{appendix:explanation}
Predictive multiplicity and the Rashomon effect are related to individual fairness \citep{dwork2012fairness, dwork2018individual}.
A mechanism $M: \calX \to \calY$ satisfies individual fairness if for every $x, x' \in \calX$, $D(M(x), M(x')) \leq d(x, x')$, where $d$ and $D$ are metrics on $\calX$ and $\calY$ respectively. 
It is also called the $(D, d)$-Lipschitz property.
Individual fairness aims to ensure that ``similar individuals are treated similarly.'' The consequence of predictive multiplicity is that \emph{the same} individual can be treated differently due to arbitrary and unjustified choices made during the training process (e.g., parameter initialization, random seed, dropout probability, etc.).
Integrating predictive multiplicity and individual fairness could results in a more thorough formulation for fair machine learning.

Predictive multiplicity allows different predictions from competing classifiers for the samples.
Thus predictive multiplicity could lead to different decision regions when training with the same dataset and achieving similar performance, and makes it hard for a machine learning practitioner to reproduce the decision regions if a different initialization of a classifier is selected.
\citet{somepalli2022can} studied the reproducibility of decision regions of almost-equally performing learning models, and observe that changes in model architecture (which reflect the inductive bias) lead to visible changes in decision regions.
Notably, neural networks with very narrows or wide layers have better reproducibility in their decision regions.
On the other hand, neural networks with ``moderate'' number of neurons in each layer have decision regions fragmented into many small pieces, and are harder to reproduce.
The connection between predictive multiplicity, neural network architectures, and inductive bias is also an interesting research direction.
For example, a stronger inductive bias could restrict the arbitrariness of a training process, leading to smaller predictive multiplicity.

The fact that  multiple classifiers may yield distinct predictions to a target a sample while having statistically identical average loss performance can also cause security issues in machine learning.
The score variation could result from a malicious learner/designer who either plants an undetectable backdoor or  carefully selects a specific model. This may result in intentional manipulation of the output scores for a sample without detectable performance changes \citep{goldwasser2022planting}.

\subsection{Predictive multiplicity with small Rashomon parameters}\label{appendix:small-epsilon}
Note that $\epsilon = 0$ implies prefect generalization to the test set, and is in general infeasible due to a limited number of samples and optimization techniques.
Moreover, a small $\epsilon$ could lead to high predictive multiplicity. 
Consider a classification task with $1000$ samples with binary classes, and trained with the 0-1 loss. 
Suppose $\epsilon = 0.001$, i.e., only allowing one sample $\bx_i$ at a time to be misclassified.
If the hypothesis space is tremendous, it is possible to find a classifier that only assign the wrong label to any $\bx_i$, and thus the ambiguity in \eqref{eq:ambiguity-discrepancy} could be 1. 

\subsection{Metrics for the spread of scores}\label{appendix:other-metrics}
The divergence measure between two distributions used in \eqref{eq:score-spread} is not restricted to KL-divergence.
For example, given a convex function $f:(0,\infty)\to \R$ satisfying $f(1)=0$, and assume that $P$ and $Q$ are two probability distributions over a set $\calX$, and $P$ is absolutely continuous with respect to $Q$.
The $f$-divergence between $P$ and $Q$ is given by \citep{csiszar1995generalized}
\begin{equation}\label{Def:f_KL}
    D_f(P\|Q)\triangleq \mathbb{E}_{Q}\left[f\left(\frac{P(X)}{Q(X)}\right) \right].
\end{equation}
Different choices of $f$ lead to different divergence; for example, 
if $f(t) = t\log t$, $D_f(P\|Q) = D_{KL}(P\|Q)$; if $f(t) = (t-1)^2$, $D_f(P\|Q) = \chi^2(P\|Q)$ is the chi-square divergence; if $f(t) = t\log t - (1+t)\log (1+t)/2$,  $D_f(P\|Q) = D_{JS}(P\|Q)$ is the Jensen-Shannon divergence.
Another example of a tunable probability divergence is the R\'enyi divergence $R_\alpha(P\|Q)$ of order $\alpha \in \R^+/\{1\}$, defined as \citep{renyi1961measures}
\begin{equation}
    D_\alpha(P\|Q) \triangleq \frac{1}{\alpha-1} \log \left( \sum_x \left(\frac{P(x)}{Q(x)}\right)^\alpha Q(x) \right),
\end{equation}
Its continuous extensions for $\alpha = 1$ and $\infty$ can also be defined.
In particular, for $\alpha = 1$, the R\'enyi divergence recovers KL divergence, and for $\alpha = \infty$, $D_\infty(P\|Q) = \max_x \log P(x)/Q(x)$ is called the max-divergence.
Both the $f$-divergence and R\'enyi divergence generalize the usual notion of KL-divergence used in this paper, and these families of divergences could also be used to measure the spread of the scores in the probability simplex. 
For example, \citet{nielsen2020generalization} reported an iterative algorithm to numerically compute a centroid for a set of probability densities measured by the Jensen–Shannon divergence.
However, these generalizations of the KL divergence do not necessarily lead to multiplicity metrics that satisfy the properties outlined in Definition~\ref{def:metric}.
More importantly, when taking the supremum over the input distributions (see~\eqref{eq:rashomon-capacity}), we are unaware of a procedure as simple as  the Blahut-Arimoto algorithm to estimate the corresponding Rashomon Capacity if the probability divergence is not the KL-divergence. Exploring alternative metrics for measuring score ``spread'' is a promising future research direction.

\subsection{Geometric interpretation of Rashomon Capacity}\label{appendix:geometric-interpretation}
In Section~\ref{sec:rashomon-capacity}, we introduce the Rashomon capacity to measure the spread of scores from a geometric viewpoint.
Here, we further discuss the pleasing geometric interpretations possessed by Rashomon Capacity, which can be found in information theory.
Particularly, given a sample $\bx_i$, let the information radius $\rad(\calM_\epsilon(\bx_i))$ and information diameter $\diam(\calM_\epsilon(\bx_i))$ of the $\epsilon$-multiplicity set $\calM_\epsilon(\bx_i)$ be \citep{wu2017lecture}
\begin{equation}\label{eq:raius-diameter}
 \rad(\calM_\epsilon(\bx_i)) = \inf\limits_{\bq\in\Delta_c} \sup\limits_{\bp \in \calM_\epsilon(\bx_i)} D_{KL}(\bp\|\bq),\; \diam(\calM_\epsilon(\bx_i)) = \sup\limits_{\bp, \bp' \in \calM_\epsilon(\bx_i)} D_{KL}(\bp\|\bp'),
\end{equation}
we have 
\begin{equation}\label{eq:raius-diameter-inequality}
    C(\calM_\epsilon(\bx_i)) \leq \rad(\calM_\epsilon(\bx_i)) \leq \diam(\calM_\epsilon(\bx_i)),
\end{equation}
where $C(\calM_\epsilon(\bx_i)) = \rad(\calM_\epsilon(\bx_i))$ if $\calM_\epsilon(\bx_i)$ is a convex set \citep{kemperman1974shannon}.

The proof of \eqref{eq:raius-diameter-inequality} is straightforward:
\begin{equation}
\begin{aligned}
    C(\calM_\epsilon(\bx_i)) 
    &= \sup\limits_{\boldsymbol{\alpha}\in\Delta_m} \inf\limits_{\bq\in\Delta_c} \sum_{j=1}^m \alpha_j D_{KL}(\bp_j\|\bq)\\
    &\leq \inf\limits_{\bq\in\Delta_c} \sup\limits_{\boldsymbol{\alpha}\in\Delta_m} \sum_{j=1}^m \alpha_j D_{KL}(\bp_j\|\bq)\\
    &= \inf\limits_{\bq\in\Delta_c} \sup\limits_{\boldsymbol{\alpha}\in\Delta_m} \E_{m\sim\boldsymbol{\alpha}} D_{KL}(P_{Y|M=m}\|\bq)\\
    &\leq  \inf\limits_{\bq\in\Delta_c} \sup\limits_{\bp \in \calM_\epsilon(\bx_i)} D_{KL}(\bp\|\bq) \triangleq \rad(\calM_\epsilon(\bx_i))\\
    &\leq \sup\limits_{\bp, \bp' \in \calM_\epsilon(\bx_i)} D_{KL}(\bp\|\bp') \triangleq \diam(\calM_\epsilon(\bx_i)).
\end{aligned}
\end{equation}

At first glance, the information radius or diameter seem to be more intuitive metrics to measure the ``spread'' of the all possible scores from the Rashomon set; however, both of them do not satisfy the properties in Definition~\ref{def:metric}.
More importantly, \eqref{eq:raius-diameter-inequality} shows that Rashomon capacity is a tighter metric, and is less likely to overestimate the spread of scores, i.e., the predictive multiplicity. 
Moreover, maximizing the KL divergence is in general an ill-posed problem since the KL divergence is (jointly) convex, and could diverge to infinity.
\citet{zhang2019theoretically} demonstrated that maximizing the KL divergence between the scores generated by a classifier with two different samples is solvable if the Euclidean distance between the two samples is upper bounded. 
In our case, we do not have  control over $\bp, \bp' \in \calM_\epsilon(\bx_i)$ and the underlying models that output $\bp, \bp'$ since two models could be very different from each other (in terms of, e.g., the Euclidean distance of the model parameters), but still yield similar test loss due to the existence of multiple local minima.

\subsection{The Blahut-Arimoto algorithm}\label{appendix:ba}
For the sake of completeness, we describe the Blahut-Arimoto (BA) algorithm \citep{blahut1972computation, arimoto1972algorithm} used in Section~\ref{sec:exp} for computing channel capacity.
For a discrete memoryless channel (DMC) $X \to Y$ with transition probabilities $P_{Y|X}$ and input probability $Q$, where $\calX= [1, \cdots, m]$ and $\calY = [1, \cdots, c]$.
The mutual information $I(X; Y)$ between $X$ and $Y$ is defined as 
\begin{equation}
    I(X; Y) \triangleq \sum_{i=1}^m \sum_{j=1}^c P_{X, Y}(i, j) \log \frac{P_{X, Y}(i, j)}{Q(i)P_Y(j)} = \sum_{i=1}^m \sum_{j=1}^c P_{Y|X}(j|i) Q(i) \log \frac{P_{X|Y}(i|j)}{Q(i)}.
\end{equation}
By definition, the capacity of the channel $P_{Y|X}$ is defined as
\begin{equation}
\begin{aligned}\label{eq:channel-capacity}
C(P_{Y|X}) &= \max\limits_{Q} I(X; Y) = \max\limits_{Q} \sum_{i=1}^m \sum_{j=1}^c P_{Y|X}(j|i) Q(i) \log \frac{P_{X|Y}(i|j)}{Q(i)},
\end{aligned}
\end{equation}
where $P_{X|Y}(i|j) = \frac{P_{Y|X}(j|i)Q(i)}{\sum_k P_{Y|X}(j|k)Q(k)}$.
Since $P_{X|Y}(i|j)$ can be viewed as a function of the channel $P_{Y|X}$ and $Q$, from \eqref{eq:channel-capacity}, it is clear that for a fixed channel $P_{Y|X}$, the channel capacity is a convex function of the input probabilities $Q$.
Denote any $P_{X|Y}(i|j) = \Phi(i|j)$, we can alternatively express the mutual information as
\begin{equation}
    I(X; Y) = \sum_{i=1}^m \sum_{j=1}^c P_{Y|X}(j|i) Q(i) \log \frac{\Phi(i|j)}{Q(i)} = J(Q, \Phi).
\end{equation}
It can be proven that \citep{blahut1972computation, arimoto1972algorithm}
\begin{enumerate}
    \item For a fixed $Q$, $J(Q, \Phi) \leq J(Q, P_{X|Y})$, i.e., $J(Q, P_{X|Y}) = \max_{\Phi} J(Q, \Phi)$, and therefore $C(P_{Y|X}) = \max_Q\max_{\Phi} J(Q, \Phi)$.
    \item For a fixed $\Phi$, $J(Q, \Phi) \leq \log \left( \sum_{i=1}^m r(i) \right)$, $r(i) = \exp\left[ \sum_{j=1}^c P_{Y|X}(j|i) \log \Phi(i|j) \right]$, where equality holds if and only if $Q(i) = r(i) / \sum_{k=1}^m r(k)$.
\end{enumerate}
The BA algorithm is built upon these two properties, and doubly maximizes $J(Q, \Phi)$.
More specifically, let $t$ be the iteration index, and let $Q^0$ be a choose initialization of the input distribution, for each iteration, we update $\Phi$ and $Q$ by
\begin{enumerate}
    \item $\Phi^{l+1}(i|j) = \frac{Q^l(i)P_{Y|X}(j|i)}{\sum_{k=1}^m Q^l(k)P_{Y|X}(j|k)}$, $\forall i, j$.
    \item $r^{l+1}(i) = \exp \left( \sum_{j=1}^c P_{Y|X}(j|i) \log \Phi^{l+1}(i|j) \right)$.
    \item $Q^{l+1}(i) = \frac{r^{l+1}(i)}{\sum_{k=1}^m r^{l+1}(k)}$.
    \item $J(Q^{l+1}, \Phi^{l+1}) = \log \left( \sum_{i=1}^m r^{l+1}(i) \right)$.
    \item $l = l + 1$.
\end{enumerate}
For the stopping criteria, let $c^{l}(i) = r^{l}(i) / Q^l(i)$, we have $J(Q^l, \Phi^l) = \log \left( \sum_{i=1}^m Q^l(i)c^{l}(i) \right)$.
Since $J(Q^l, \Phi^l)$ is the logarithm of the average of $c^{l}(i)$, we have
\begin{equation}
    \log \left( \sum_{i=1}^m Q^l(i)c^{l}(i) \right) \leq C(P_{Y|X}) \leq \max_i \log c^{l}(i),
\end{equation}
and therefore we update $Q^{l+1}(i)$ and $\Phi^{l+1}(i|j)$ until the stopping criteria is matched,
\begin{equation}
    \max_i \log c^{l}(i) - \log \left( \sum_{i=1}^m Q^l(i)c^{l}(i) \right) \leq \epsilon,
\end{equation}
where $\epsilon > 0$ is a pre-defined accuracy parameter.

The BA algorithm has also been extended to channels with continuous input and output alphabets, i.e., $|\calX| = \infty$ and $|\calY| = \infty$, based on sequential Monte-Carlo integration methods (i.e., particle filters) \citep{dauwels2005numerical, cao2013capacityp1, cao2013capacityp2, farsad2020capacities}. Since we deal with finite predicted classes and discrete Rashomon sets, Proposition~\ref{prop:sampling-rashomon-set} allows us to circumvent the use of more sophisticated variations of the BA algorithm.

\subsection{Adversarial weight perturbation on unregularized logistic regression}\label{appendix:awp-logistic}
In \eqref{eq:search-parameter}, we introduce an adversarial weight perturbation procedure to estimate Rashomon Capacity in the Rashomon set. 
In general, the problem in \eqref{eq:search-parameter} is difficult to analyze, and is usually optimized by using automated gradient computation tools such as Tensorflow \citep{abadi2016tensorflow}.
Here, we provide a special case of unregularized logistic regression, which gradient and Hessian can be analytically computed.
We start with a more general case by considering a set of features and labels $\{\bz_i, y_i\}_{i=1}^n$ for a binary classification problem, where $\bz_i \in \R^m$ and $y_i \in \{0, 1\}$.
Logistic regression assumes the output scores 
\begin{equation}
\begin{aligned}
    P(\hat{Y} = 1|Z=\bz_i; \bw) &= \frac{e^{\bz_i^\top \bw}}{1+e^{\bz_i^\top \bw}}\;\text{and}\\
    P(\hat{Y} = 0|Z=\bz_i; \bw) &= 1- P(\hat{Y} = 1|Z=\bz_i; \bw) = \frac{1}{1+e^{\bz_i^\top \bw}},
\end{aligned}
\end{equation}
where $\bw\in \R^m$ is the vector of weights.
The loss in logistic regression (without regularization) is defined as \citep{hastie2009elements}
\begin{equation}
\begin{aligned}
    \ell(\bw) &= - \sum_{i=1}^n \left( y_i \log P(\hat{Y} = 1|Z=\bz_i; \bw) + (1-y_i) \log (1-P(\hat{Y} = 1|Z=\bz_i; \bw)) \right)\\
    &= \sum_{i=1}^n \left( y_i\bw^\top\bz_i - \log (1+e^{\bw^\top \bz_i}) \right).
\end{aligned}
\end{equation}
Let $\bZ = [\bz_1, \cdots, \bz_n]^\top \in \R^{n\times m}$ be the feature matrix, $\by = [y_1, \cdots, y_n]^\top$ the label vector, $\bp = [P(\hat{Y} = 1|Z=\bz_1; \bw), \cdots, P(\hat{Y} = 1|Z=\bz_n; \bw)]$ be the score vector, and $\bW = \diag(\bw_1, \cdots, \bw_m)$ be the weight matrix with diagonal entries equal to $\bw$.
The gradient and Hessian of $\ell(\bw)$ with respect to $\bw$ can be expressed as
\begin{equation}\label{eq:logistic-gradient}
\begin{aligned}
\nabla \ell(\bw) &= -\sum_{i=1}^n \bz_i (y_i - P(\hat{Y} = 1|Z=\bz_i; \bw)) = \bZ^\top (\by - \bp),\;\text{and}\\
\nabla^2 \ell(\bw) &= -\sum_{i=1}^n \bz_i\bz_i^\top P(\hat{Y} = 1|Z=\bz_i; \bw)(1-P(\hat{Y} = 1|Z=\bz_i; \bw)) = -\bZ^\top \bW \bZ.
\end{aligned}
\end{equation}
We use the Newton–Raphson algorithm to update the weights, i.e.,
\begin{equation}
\begin{aligned}
    \bw^{t+1} &= \bw^t - \left( \nabla^2 \ell(\bw) \right)^{-1}\nabla \ell(\bw)\\
    &= \bw^t + \left( \bZ^\top \bW \bZ \right)^{-1} \bZ^\top (\by - \bp),
\end{aligned}
\end{equation}
where $t \in [1, T]$ is the index of the iterations, and $\bw^t$ is the weight at iteration t.
Note that the features $\bz_i$ could be kernel transformation of a sample $\bx_i$, logits outputed from a neural network of a sample $\bx_i$, or even the sample $\bx_i$ itself. When $\bz_i = \bx_i$, it is the vanilla logistic regression.

In order to perform adversarial weight perturbation on $\bw$ (i.e., to maximize scores of different classes in \eqref{eq:search-parameter}), for a target feature input $\bz_t$, when $y_t = 0$, we aim to maximize $\bw^\top \bz_t$ such that $P(\hat{Y} = 1|Z=\bz_i; \bw)$ is maximized.
Similarly, when $y_t = 1$, we aim to minimize $\bw^\top \bz_t$ such that $P(\hat{Y} = 0|Z=\bz_i; \bw)$ is maximized.
Therefore, we modify the gradient in \eqref{eq:logistic-gradient} to
\begin{equation}
\begin{aligned}\label{eq:adversarial-logistic}
    \nabla \ell(\bw) = \bZ^\top (\by - \bp) + \lambda_t \bz_t,
\end{aligned}
\end{equation}
where $\lambda_t$ is a regularization parameter, and $\lambda_t > 0$ if $y_t = 0$, and $\lambda_t < 0$ if $y_t = 1$. When $\lambda = 0$, \eqref{eq:adversarial-logistic} degenerates to \eqref{eq:logistic-gradient}.
Therefore, the adversarial weight perturbation on logistic regression could be performed by keep updating the weights with
\begin{equation}\label{eq:adversarial-update}
    \bw^{t+1} = \bw^t + \left( \bZ^\top \bW \bZ \right)^{-1}\left( \bZ^\top (\by - \bp) + \lambda_t \bz_t \right),
\end{equation}
until convergence.
The reason we introduce the features $\bz_i$ in the beginning instead of the samples $\bz_i$ if that if $\bz_i = f(\bx_i)$ for a neural network $f(\cdot)$, \eqref{eq:adversarial-update} can be used for last-layer weight perturbation of the neural network \citep{tsai2021formalizing}. 

\begin{figure}[!tb]
\centering
\includegraphics[width=.8\textwidth]{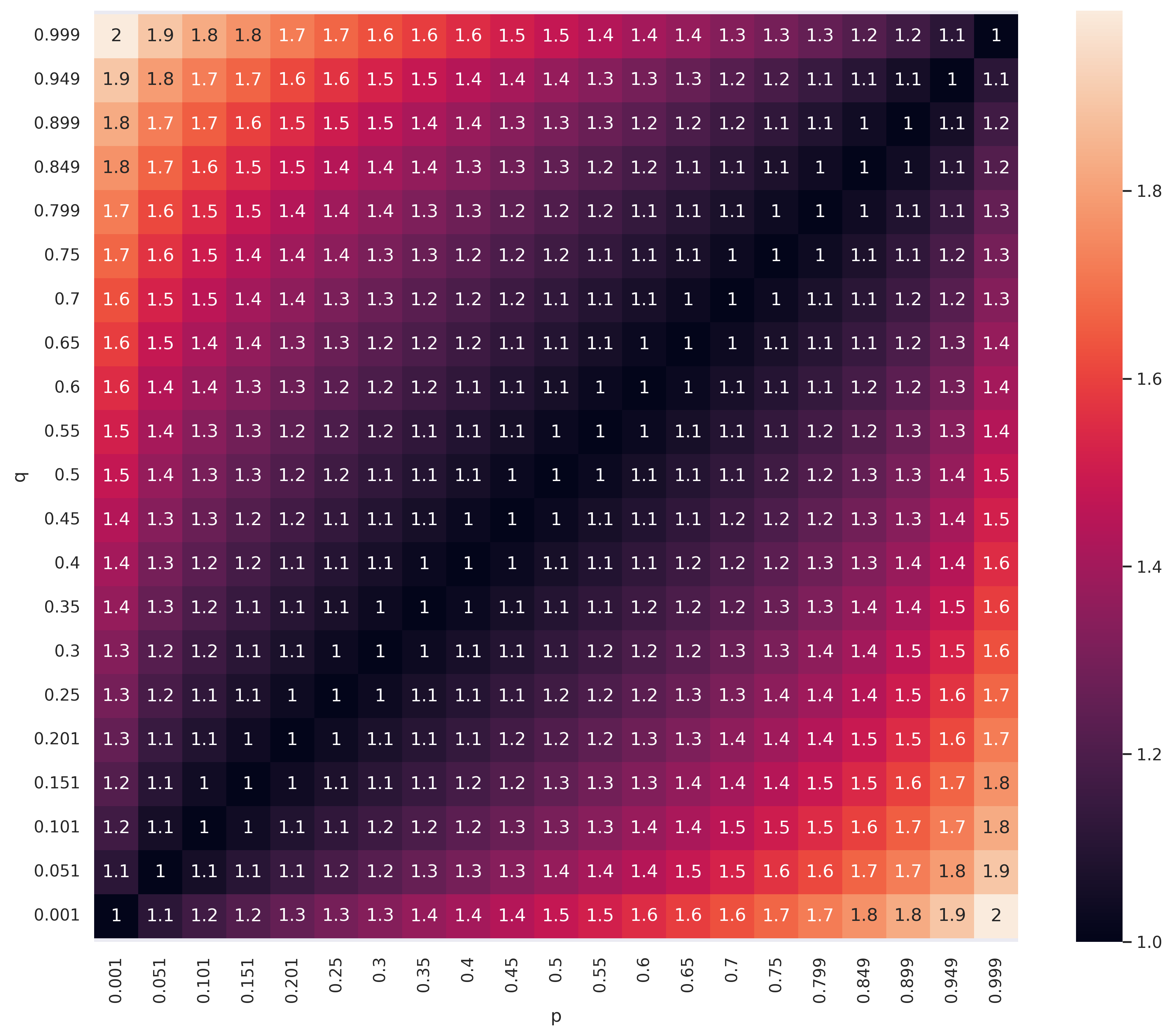}
\caption{Channel capacity (values annotated on the heatmap) of the binary asymmetric channel with different $p$ and $q$.}
\label{fig:capacity-convexity}
\end{figure}

\subsection{Convexity of the channel capacity}\label{appendix:convexity}
In the last paragraph of Section~\ref{sec:exp}, we mention an important limitation of KL-divergence based Rashomon Capacity due to the convexity of KL-divergence: in certain cases $C(\calM_\epsilon(\bx_i))$ (and therefore $m_C(\bx_i) = 2^{C(\calM_\epsilon(\bx_i)) }$) may seem small for already significant score variations across the classes.
Here, we use an example the binary asymmetric channel \citep{cover1999elements} to illustrate this phenomenon.
Given $p, q \in [0, 1]$, a binary asymmetric channel $X\to Y$ has a channel transition matrix $\bP = [[p, 1-p], [q, 1-q]] \in [0, 1]^{2\times 2}$. When $p = q$, the binary asymmetric channel matches the binary symmetric channel.
In Fig.~\ref{fig:capacity-convexity}, we show the channel capacity algorithm, with different pairs $(p, q)$. We observe that the channel capacity is a very ``flat'' convex function of $p$ and $q$; for example, when $p = 0.5$ and $q = 0.1$, the channel capacity is 1.3, and the channel capacity is larger than 1.8 if the difference $|p-q|$ is larger than 0.7.
When the channel transition matrix to be the estimated scores in a binary classification problem, Rashomon Capacity must be interpreted accordingly
For example, the difference of the scores of a sample for class 0 and class 1 needs to be larger than 0.7 such that the Rashomon Capacity exceeds 1.8. In fact, a Rashomon capacity above 1.1  already corresponds to a potentially significant score variation in practice.

\section{Datasets and experiments setups}\label{appendix:exp}
\subsection{Dataset descriptions and pre-processing procedures.}\label{appendix:data-preprocessing}
\paragraph{UCI adult dataset.}
The UCI Adult dataset \citep{Lichman:2013} contains multiple domestic factors including an individual’s education level, age, gender, occupation, and etc.
We drop missing values, and obtain 46447 samples with 20 features. 
The 20 features include the one-hot encoded version of the originally selected features [age, education, marital-status, relationship, race, gender, capital-gain, capital-loss, hours-per-week]; note that the features 'race' and 'gender' are binarized.
The label is the income, and is divide into two classes: <=50K and >50K.

\paragraph{COMPAS recidivism dataset.}
The COMPAS (Correctional Offender Management Profiling for Alternative Sanctions) dataset \citep{angwin2016machine} is a widely used algorithm for judges and parole officers to score criminal defendant’s likelihood of reoffending (i.e., recidivism).
The features include [age, charge degree, race, sex, priors crime count, days before screening/arrest, jail in date, jail out date], and the label is the binary prediction on recidivism.
We pre-processed the features by binarizing 'race', 'sex', 'charge degree' (felony or others), and 'days before screening/arrest' (<= 30 days or > 30 days); creating a new feature call 'length of stay', which is duration between 'jail in date' and 'jail out date'.
The resulting dataset has 52878 samples with 6 features.

\paragraph{HSLS dataset.}
The HSLS (High School Longitudinal Study) dataset \citep{ingels2011high} is collected from 23,000+ participants across 944 high schools in the USA, and it includes thousands of features such as student demographic information, school information, and students' academic performance across several years. 
We pre-processed the dataset (e.g., dropping rows with a significant number of missing entries and students taking repeated exams, performing k-NN imputation, normalization), and the number of samples reduced to 14,509 and the number of features is 59.
For the labels, we created a binary label Y from students’ $9^\text{th}$-grade math test score (i.e., top 50\% vs. bottom 50\%).

\paragraph{CIFAR-10 dataset.} 
The CIFAR-10 dataset \citep{krizhevsky2009learning} contains 50,000 colored images for training and 10,000 for test, where each images has 32 × 32 pixels, and has a label 10 classes [airplanes, cars, birds, cats, deer, dogs, frogs, horses, ships, and trucks].
The samples are distributed evenly on the 10 classes for both training and test set.

\subsection{Training details and experimental setups}\label{appendix:exp-setup}
For UCI Adult, COMPAS and HSLS datasets, the hypothesis space is composed of simple feed-forward neural networks with ReLU activations, and the optimizer is gradient descent trained with the whole datasets, and the training loss is the cross-entropy loss, and the learning rate is 0.001. 
For UCI Adult dataset, the neural networks have 5 layers/100 neurons per layer, and is trained with 100 epochs.
For COMPAS dataset, the neural networks have 5 layers/200 neurons per layer, and is trained with 200 epochs.
For HSLS dataset, the neural networks have 5 layers/200 neurons per layer, and is trained with 500 epochs.

For CIFAR-10 dataset, the hypothesis space is composed of VGG16 convolutional neural networks  \citep{simonyan2014very}, and the optimizer is stochastic gradient descent with batch size 40.
The VGG16 models are trained with the cross-entropy loss for 3 epochs and the learning rate is 0.001. 

\paragraph{Sampling.}
For UCI Adult, COMPAS and HSLS datasets, we did 5 repeated experiments with difference random seeds for 70\%/30\% train/test split, and in each experiments, we trained 100 models, and evaluated on the test set.
We select the smallest test loss, and select models that have test losses smaller than the smallest test loss plus the Rashomon parameter $\epsilon = [0.01, 0.02, 0.05, 0.1]$.
For CIFAR-10 dataset,  we did 2 repeated experiments with difference random seeds for 90\%/10\% train/test split, and in each experiments, we trained 50 models, and evaluated on the test set.
The mean accuracy for UCI Adult, COMPAS, HSLS and CIFAR-10 datasets are $0.8034$, $0.6540$, $0.6247$ and $0.8380$ respectively.
The Rashomon Capacity of all test samples can then be computed by the scores generated by the selected models for difference $\epsilon$, and the mean and standard errors of the largest 1\% and 5\% Rashomon Capacity, i.e., the statistics on the tails of the Rashomon Capacity, are reported in Fig.~\ref{fig:capacity-profiles}.

\paragraph{Adversarial weight perturbation.}
For UCI Adult, COMPAS and HSLS datasets, we did 3 repeated experiments with difference random seeds for 95\%/5\% train/test split, 90\%/10\% train/test split and 90\%/10\% train/test split respectively. 
We first trained a base classifier, and perturbed the weights of the neural networks for each test sample (cf.~\eqref{eq:search-parameter}) with learning rates 0.001 (for UCI Adult and COMPAS datasets) and 0.01 (for HSLS dataset). 
We require the perturbation procedure to stop updating the weights if either the perturbed scores exceed 0.9, or the test loss is larger than the base test loss plus the Rashomon parameter $\epsilon = [0.01, 0.02, 0.05, 0.1]$.
Similarly, for CIFAR-10 dataset, we did 2 repeated experiments with difference random seeds for 99\%/1\% train/test split.
The mean accuracy of the base classifiers for UCI Adult, COMPAS, HSLS and CIFAR-10 datasets are $0.8028$, $0.6458$, $0.7039$ and $0.8167$ respectively.
Therefore, for each sample, we computed the Rashomon Capacity of all test samples with scores from the base classifier and from the perturbed classifier.

\begin{algorithm}
\caption{Sampling with Rejection}\label{alg:sampling-RS}
\begin{algorithmic}
\Require training set $\calS$, test set $\calT$, number of models $m\in\mathbb{N}$, Rashomon parameter $\epsilon > 0$
\State SampledModel $\gets [~]$ 
\State TestLoss $\gets [~]$ 
\State RashomonSet $\gets [~]$ 
\State RashomonSetProb $\gets [~]$ 
\For{$i \in [m]$}
\State model $\gets$ train($\calS$, random\_seed=i)
\State loss $\gets$ evaluate(model, $\calT$)
\State SampledModel.append(model)
\State TestLoss.append(loss)
\EndFor
\For{$i \in [m]$}
\If{TestLoss[i] < min(TestLoss) + $\epsilon$}
\State RashomonSet.append(SampledModel[i])
\State RashomonSetProb.append([compute\_scores(SampledModel[i], $\calT$)])
\EndIf
\EndFor
\State \Return RashomonSet, RashomonSetProb
\end{algorithmic}
\end{algorithm}

\begin{algorithm}
\caption{Adversarial Weight Perturbation (AWP)}\label{alg:awp-RS}
\begin{algorithmic}
\Require dataset $\calD = \{\bx_i, \by_i\}_{i=1}^n$, pretrained model $f_\theta: \calX\to\Delta_c$ with weight $\theta$,  learning rate $\gamma$,  number of classes $c$
\State RashomonSetProb $\gets$ zeros(n, c, c)
\State BaseLoss $\gets$ evaluate($f_\theta$, $\calD$)
\For{$i \in [n]$}
\For{$j \in [c]$}
\State CurrentLoss $\gets$ evaluate($f_\theta$, $\calD$)
\While{CurrentLoss < BaseLoss + $\epsilon$}
\State scores $\gets$ $f_\theta(\bx_i)$
\State $\nabla \theta \gets \frac{\partial\text{-scores[j]}}{\partial \theta}$ 
\State $\theta \gets \theta + \gamma \nabla \theta$
\State CurrentLoss $\gets$ evaluate($f_\theta$, $\calD$)
\EndWhile
\State RashomonSetProb[i, j, :] $\gets$ $f_\theta(\bx_i)$
\EndFor
\EndFor
\end{algorithmic}
\end{algorithm}

\subsection{Algorithm boxes}
For the sake of clarify, we summarized the sampling and AWP algorithms to explore the Rashomon set and to compute Rashomon Capacity in Algorithm~\ref{alg:sampling-RS} and Algorithm~\ref{alg:awp-RS} respectively.
Both algorithms produce scores from models in the Rashomon set, where the scores are used later to compute the Rashomon Capacity by the Blahut-Arimoto algorithm (Section~\ref{appendix:ba}).

\section{Additional experiments}
\subsection{A Case study on the COMPAS dataset}\label{appendix:compas-case-study}
We trained 1k multi-layer perceptron classifiers with different random seeds and selected classifiers which have loss smaller than $0.685$ (the smallest loss observed was $0.68$), and compute Rashomon Capacity.
We show the samples with  Rashomon Capacity higher than 1.2 in the COMPAS dataset in Table~\ref{table:compas-case-study}.
Observe that the samples with conflicting scores are mostly with sex 1 (marked as Male in the dataset) and  numerous prior convictions (i.e., prior counts or length of stay). Thus, one must recommend caution when evaluating input samples with this profile. 
This example showcases how a stakeholder can zoom into samples with high Rashomon Capacity and flag them for further investigation.

\begin{table}[!t]
  \caption{Samples with high Rashomon Capacity in the COMPAS dataset.}
  \label{table:compas-case-study}
  \centering
  \begin{tabular}{cccccccc}
    \toprule
    Age     & Charge Degree     & Race & Sex & Prior Counts & Length of Stay & Max Score & Min Score \\
    \midrule
    25 & 1 & 0 & 1 & 5 & 35 & 0.498 & 0.194\\ 
    49 & 0 & 1 & 0 & 2 & 82 & 0.886 & 0.270 \\ 
    58 & 0 & 1 & 1 & 2 & 83 & 0.885 & 0.224 \\ 
    45 & 0 & 0 & 1 & 20& 46 & 0.354 & 0.030 \\ 
    40 & 0 & 0 & 1 & 24& 1 & 0.453 & 0.047 \\ 
    25 & 1 & 1 & 1 & 5 & 101 & 0.756 & 0.196 \\
    45 & 0 & 0 & 0 & 9 & 75 & 0.799 & 0.162 \\ 
    66 & 1 & 0 & 1 & 33& 13 & 0.489 & 0.014 \\ 
    37 & 0 & 0 & 1 & 3 & 80 & 0.867 & 0.301 \\
    58 & 1 & 1 & 1 & 7 & 185 & 0.987 & 0.343 \\ 
    53 & 1 & 1 & 1 & 9 & 117 & 0.890 & 0.199 \\ 
    29 & 0 & 0 & 1 & 1 & 99 & 0.930 & 0.416\\
    37 & 0 & 1 & 1 & 5 & 82 & 0.849 & 0.272\\ 
    37 & 0 & 0 & 1 & 22& 1 & 0.434 & 0.058\\ 
    52 & 1 & 0 & 1 & 7 & 117 & 0.921 & 0.251\\
    \bottomrule
  \end{tabular}
\end{table}

\subsection{Other methods to explore the Rashomon sets}\label{appendix:flipping-fgsm}
The procedure in \eqref{eq:search-parameter} reveals a desirable property of a Rashomon subset: it should include models with significant score variations. 
Similarly, a desirable Rashomon subset for accurately evaluating ambiguity/discrepancy is one with models that have most score disagreement.
Based on the observation, we briefly overview next  two alternative strategies for identifying a Rashomon subset: training with label flipping \citep{xiao2012adversarial} and fast gradient sign method (FSGM) \citep{goodfellow2014explaining}. 

\paragraph{Training with label flipping}
For training with label flipping, a classifier is trained with a sample whose label is adversarially corrupted to different classes, i.e., training $c$ classifiers for a sample of a $c$-class classification problem, with the goal of producing conflicting scores of the sample.
In Fig.~\ref{fig:explore-rashomon} (Left), we performed the label flipping procedure and report Rashomon Capacity with different Rashomon parameters $\epsilon$ on 1k random samples in the test set of COMPAS and HSLS datasets.
The accuracy of the base classifier and the mean accuracy of the classifiers trained with a flipped label are $0.68029$ and  $0.6660$ for COMPAS dataset, and $0.7337$ and $0.7324$ for HSLS dataset.
The Rashomon Capacity are all small since a miss classification of a single sample does not significantly influence the overall empirical risk.
Therefore, the classifiers are more likely to ignore the sample with a flipped label.

\paragraph{Fast gradient sign method (FSGM)}
The FSGM is different from the proposed adversarial weight perturbation in \eqref{eq:search-parameter} in two aspects. 
First, FSGM is applied to create an imperceivable perturbation on the samples instead of the weights.
Second, FSGM only uses the \emph{sign} of the gradient (times a scalar $\beta$) to update the weights.
We implemented FSGM on the weights to adversarially change the scores of 1k random samples in the test set of COMPAS and HSLS datasets, and report Rashomon Capacity in Fig.~\ref{fig:explore-rashomon} (Right).
Note that even with a small scalar $\beta = 0.0001$, the update on the weights---despite being imperceivable when added to input samples---could lead to a significant change of the loss when added to the model weights, and most classifiers updated with the FSGM would not belongs to the Rashomon set defined by the Rashomon parameter.
Therefore, Rashomon Capacity is almost 0, as observed in Fig.~\ref{fig:explore-rashomon} (Right).

\begin{figure}[t!]
     \centering
     \begin{subfigure}[b]{0.48\textwidth}
         \centering
         \includegraphics[width=\textwidth]{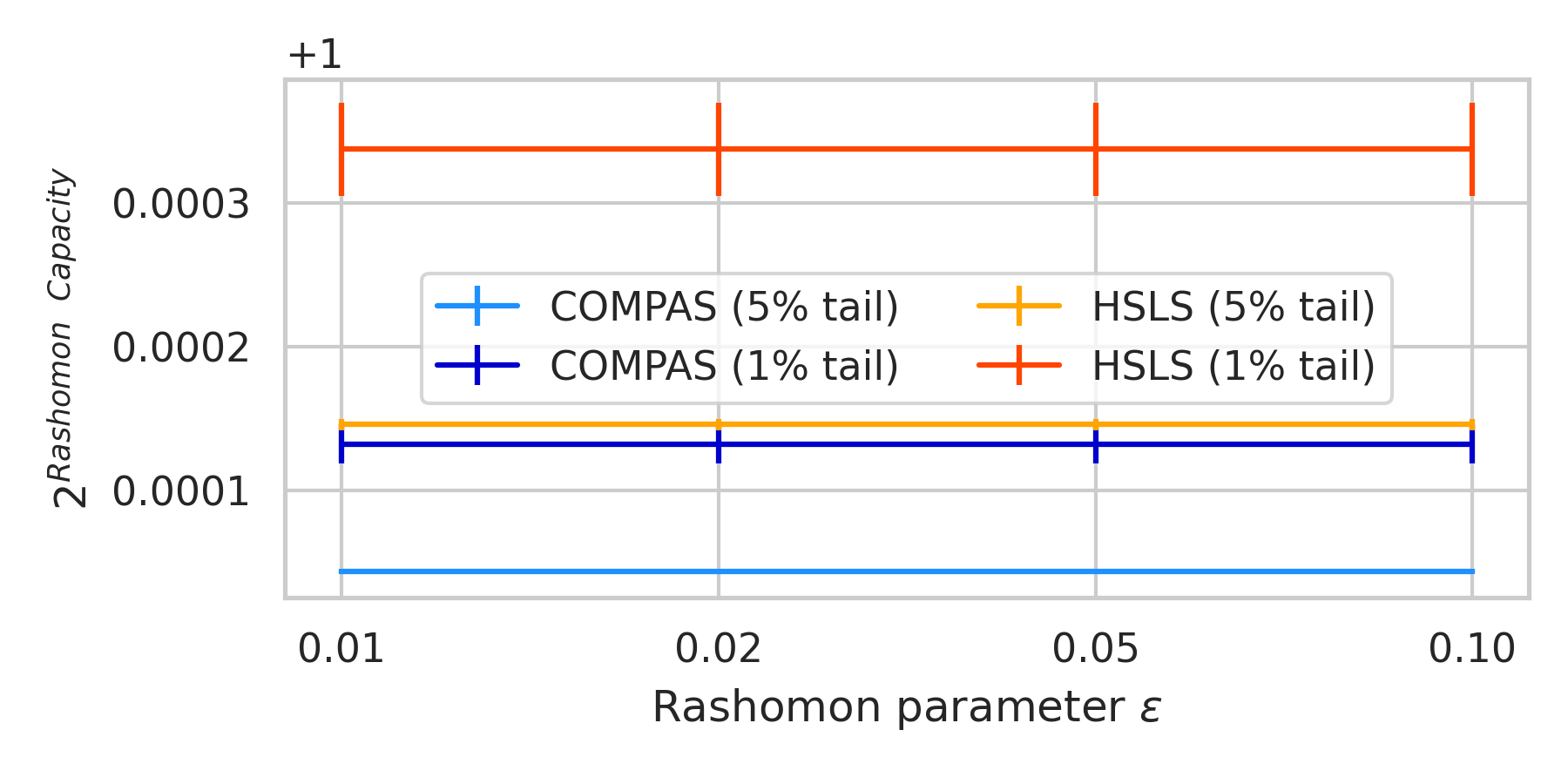}
         \caption{Training with label flipping}
     \end{subfigure}
     \begin{subfigure}[b]{0.48\textwidth}
         \centering
         \includegraphics[width=\textwidth]{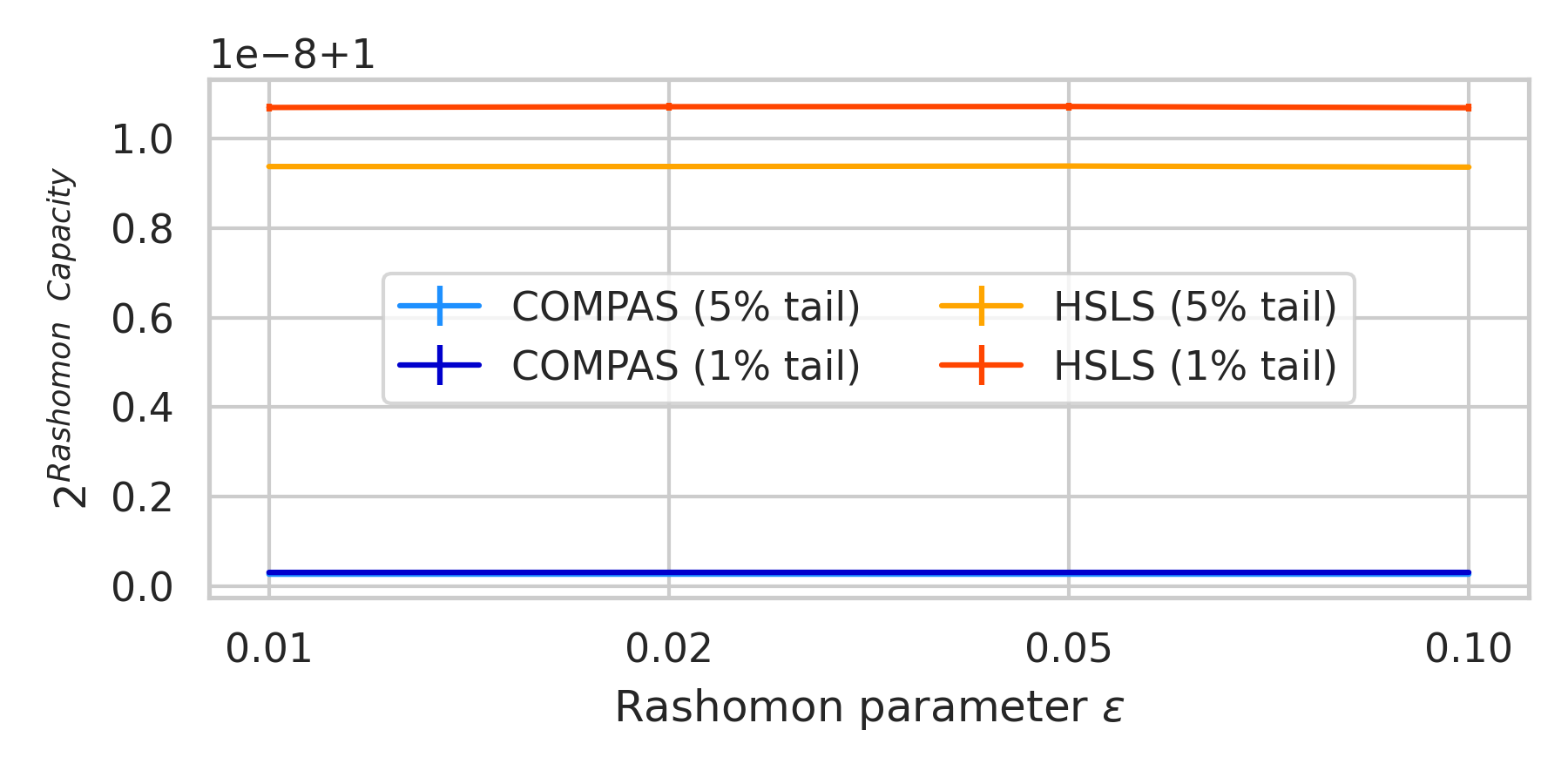}
         \caption{Fast Gradient Sign Method (FSGM)}
     \end{subfigure}
    \caption{Other methods to explore the Rashomon sets.}
    \label{fig:explore-rashomon}
\end{figure}

\subsection{Evaluating Rashomon Capacity in the decision domain}\label{appendix:rc-decision}
In Figure~\ref{fig:cifar10-score-decision}, we demonstrate that Rashomon Capacity can be evaluated with both scores and decisions generated from 100 models in the Rashomon set. 
Decision-based Rashomon Capacity has integer values, indicating the number of classes that are confused for each sample. 

\begin{figure}[!tb]
\centering
\includegraphics[width=\textwidth]{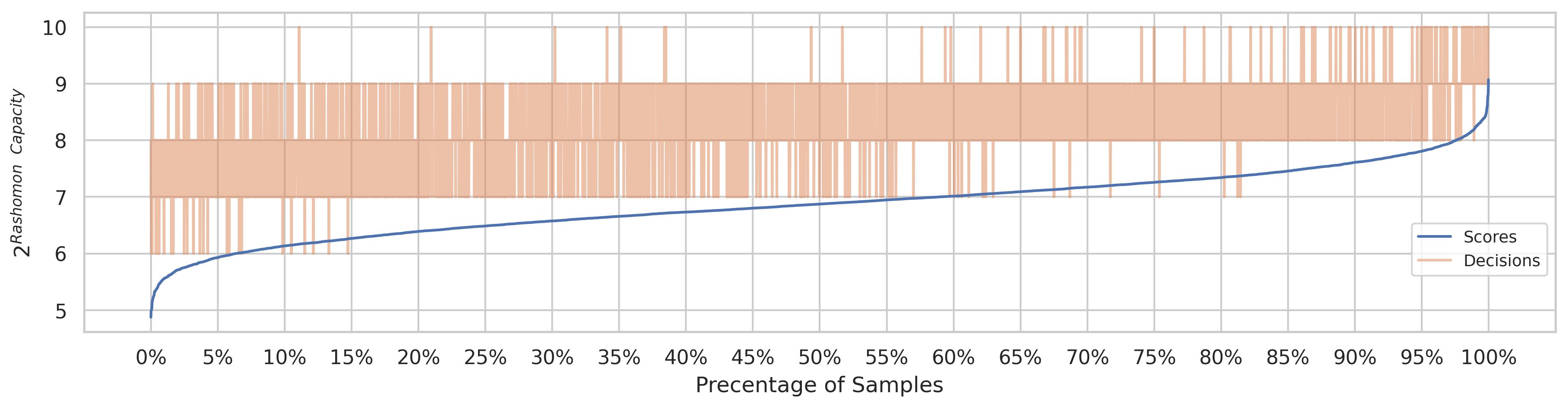}
\caption{Rashomon Capacity of samples in CIFAR-10 in score and decision domains. The samples are sorted in increasing order of score-based RC. For each sample, we then plot the RC using the corresponding thesholded scores.}
\label{fig:cifar10-score-decision}
\end{figure}

\subsection{Predictive multiplicity scores based on predicted classes: ambiguity and discrepancy}\label{appendix:ambiguity-discrepancy}

The computation of the ambiguity and discrepancy in \eqref{eq:ambiguity-discrepancy} requires searching over the entire Rashomon set, which is computationally infeasible when the hypothesis space is composed of neural networks. 
However, we can restrict the search in the entire Rashomon set to the sampled Rashomon set, and approximate the ambiguity and discrepancy.
In Fig.~\ref{fig:adult-compas-hsls-amb-dis}, we report both ambiguity and discrepancy of the 100 sampled models used to produce Fig~\ref{fig:capacity-profiles} for UCI Adult, COMPAS, and HSLS datasets.
Note that both ambiguity and discrepancy report high predictive multiplicity.
For example, in COMPAS dataset, 38\% of the samples can be assigned conflicting predictions by switching between classifiers with test loss difference less than 0.05, and in HSLS dataset, the proportion goes to 50\%.
The reason for such a high predictive multiplicity measured by the ambiguity and discrepancy is that the classifiers, despite having high accuracy, often produce similar scores across the classes for most of the samples.

\begin{figure}[!tb]
\centering
\includegraphics[width=\textwidth]{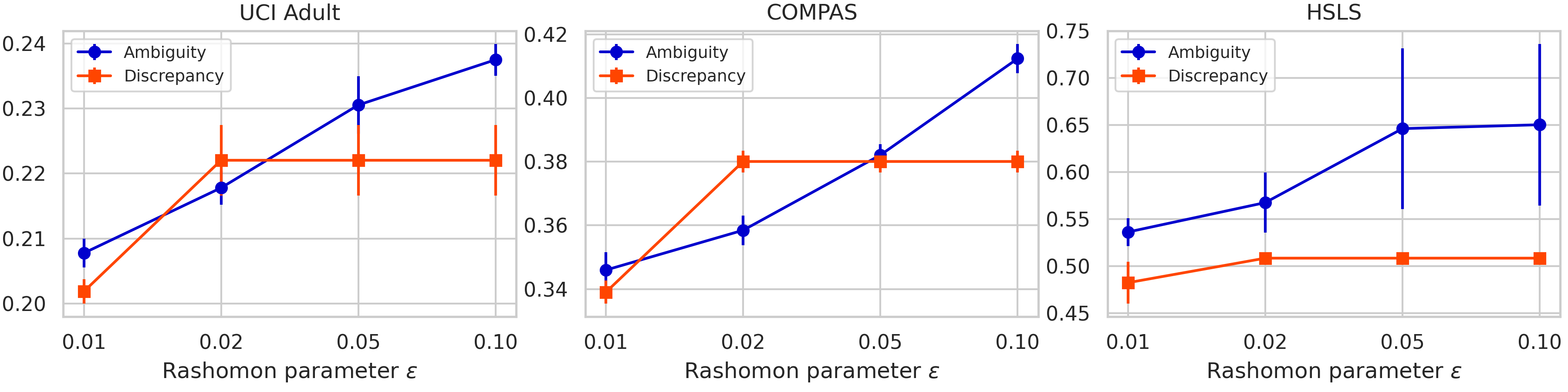}
\caption{Ambiguity and discrepancy of UCI Adult, COMPAS, and HSLS datasets.}
\label{fig:adult-compas-hsls-amb-dis}
\end{figure}

\subsection{Rashomon Capacity with ensemble methods and calibrated classifiers}\label{appendix:ensemble-calibration}
We implement Platt scaling \citep{platt1999probabilistic} to calibrate the scores produced by classifiers (decision tree and random forest classifiers), and compare the Rashomon Capacity with and without calibration in Table~\ref{table:calibration}.
We observe that with calibration, the Rashomon Capacity are slightly reduced when using decision classifiers, and slightly increased when using random forest classifiers (in UCI Adult and HSLS datasets). 
It indicates that model calibration could be a (partial) solution to reduce multiplicity, and is an interesting future direction.

We would like to further clarify that if a perfectly calibrated classifier assigns a 50\% score to a sample (e.g., in binary classification), it does not necessarily mean that this sample has high multiplicity. A perfectly calibrated classifier is one whose predicted classes matches the true classes on average across samples (e.g., samples predicted to be 50\% of one class have true outcomes matching that class 50\% of the time). However, this does not necessarily translate to a (in)consistent set of predictions for a single target sample across equally calibrated classifiers. It may be the case that all calibrated models drawn from the Rashomon Set assign the same 50\% probability for that sample (no multiplicity). Conversely, some models may assign higher and lower confidence for that sample (high multiplicity) yet, on average, still be well-calibrated. Again, this happens because calibration (like accuracy) is an average metric across all samples.

we observe that a random forest classifier leads to a smaller Rashomon Capacity when compared to a decision tree. This is likely due to random forests being an ensemble method. This is in line with two observations: (i) loss functions are often convex, so convex combinations of classifiers will not increase loss, and (ii) score variation (as measured by RC) is captured by at most $c$ models, so a small number of models can reflect multiplicity across the whole Rashomon set. Ensembling is a viable strategy for resolving multiplicity in small models, but may be infeasible for large, computationally expensive models (e.g., neural networks). We will include ensembling as a promising strategy in the Final Remark.

\begin{table}[!t]
  \caption{Rashomon Capacity with and without model calibration.}
  \label{table:calibration}
  \centering
  \begin{tabular}{cccccc}
    \toprule
    & & \multicolumn{2}{c}{Rashomon Capacity} & \multicolumn{2}{c}{Rashomon Capacity (Cali.)}                  \\
    \cmidrule(r){3-4}\cmidrule(r){5-6}
    Dataset     & Classifier     & $5\%$ tail & $1\%$ tail & $5\%$ tail & $1\%$ tail \\
    \midrule
    \multirow{2}{*}{UCI Adult} & Decision Tree        & $1.37\pm 0.16$ & $1.72\pm 0.02$ & $1.38\pm 0.16$ & $1.68\pm 0.01$\\
                               & Random Forest        & $1.00\pm 0.00$ & $1.01\pm 0.00$ & $1.03\pm 0.03$ & $1.07\pm 0.02$\\
    \multirow{2}{*}{COMAS}     & Decision Tree        & $1.04\pm 0.00$ & $1.05\pm 0.00$ & $1.00\pm 0.00$ & $1.00\pm 0.00$\\
                               & Random Forest        & $1.00\pm 0.00$ & $1.00\pm 0.00$ & $1.00\pm 0.00$ & $1.00\pm 0.00$\\
    \multirow{2}{*}{HSLS}      & Decision Tree        & $1.22\pm 0.02$ & $1.25\pm 0.02$ & $1.19\pm 0.02$ & $1.21\pm 0.01$\\
                               & Random Forest        & $1.00\pm 0.00$ & $1.00\pm 0.00$ & $1.03\pm 0.01$ & $1.04\pm 0.01$\\                      
    \bottomrule
  \end{tabular}
\end{table}

\subsection{Training without neural networks: UCI Adult, COMPAS, and HSLS datasets}\label{appendix:wo-neural-networks}
We report Rashomon Capacity with learning models that are not neural networks; particularly, we adopt decision tree/ random forest classifiers, and logistic classifiers with no, $\ell_1$, $\ell_2$ and elastic net regularizations, trained with UCI Adult, COMPAS, and HSLS datasets.
We sampled 100 classifiers for each model, and report the distribution of Rashomon Capacity among the samples from Fig.~\ref{fig:adult-dt} to Fig.~\ref{fig:hsls-lrl2} (UCI Adult: Fig.~\ref{fig:adult-dt} - Fig.~\ref{fig:adult-lrl2}; COMPAS: Fig.~\ref{fig:compas-dt} - Fig.~\ref{fig:compas-lrl2}; HSLS: Fig.~\ref{fig:hsls-dt} - Fig.~\ref{fig:hsls-lrl2}).

We observe that for all datasets, Rashomon Capacity is significantly reduced with random forest classifiers, comparing to the decision tree classifiers, i.e., predictive multiplicity is alleviated by ensemble methods with a multitude of decision trees methods. 
On UCI Adult and HSLS datasets, we observe that regularization for logistic regression could also reduce Rashomon Capacity.
These preliminary numerical results could serve as future directions on the study of reducing predictive multiplicity via ensemble methods and weight regularization.

\clearpage
\begin{figure}[!tb]
\centering
\includegraphics[width=.8\textwidth]{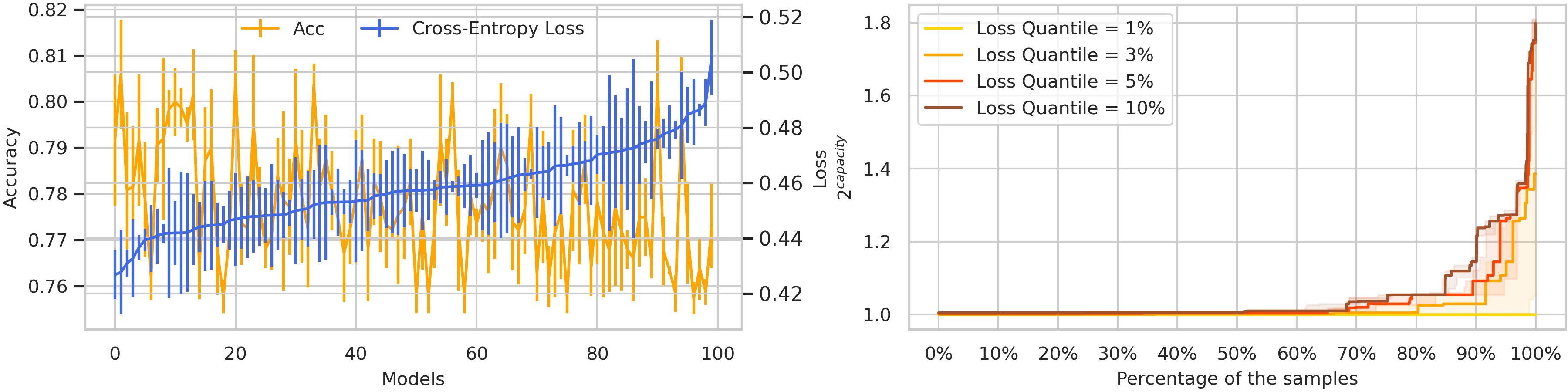}
\caption{UCI Adult dataset with decision tree classifiers.}
\label{fig:adult-dt}
\end{figure}

\begin{figure}[!tb]
\centering
\includegraphics[width=.8\textwidth]{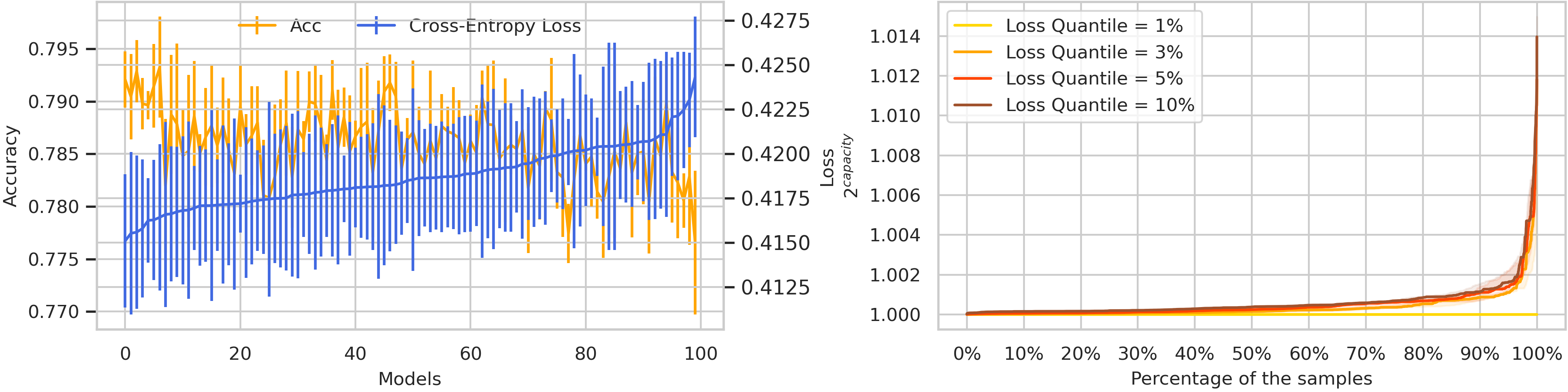}
\caption{UCI Adult dataset with random forest classifiers.}
\label{fig:adult-rf}
\end{figure}

\begin{figure}[!tb]
\centering
\includegraphics[width=.8\textwidth]{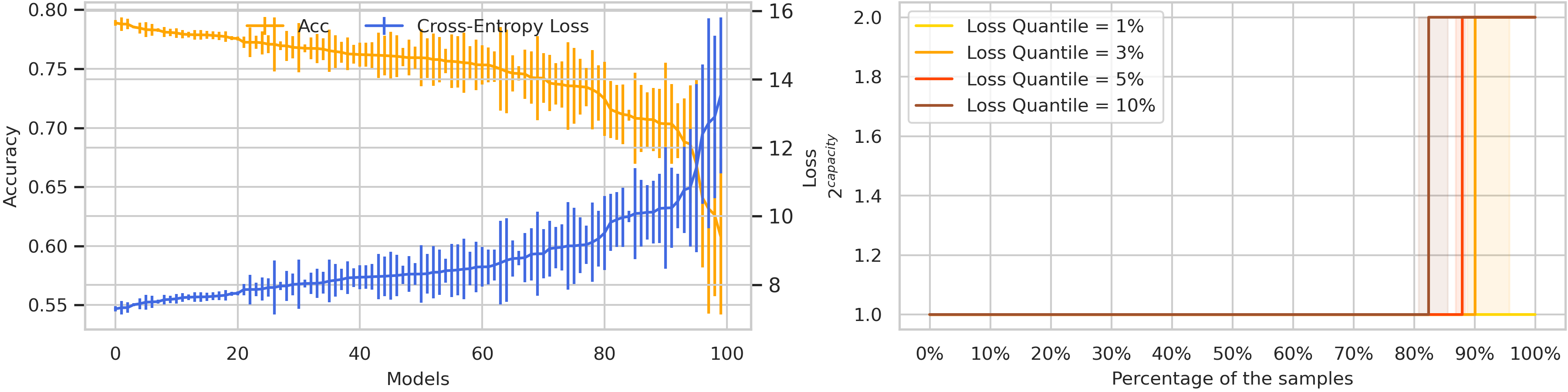}
\caption{UCI Adult dataset with logistic regression and no regularization.}
\label{fig:adult-lr}
\end{figure}

\begin{figure}[!tb]
\centering
\includegraphics[width=.8\textwidth]{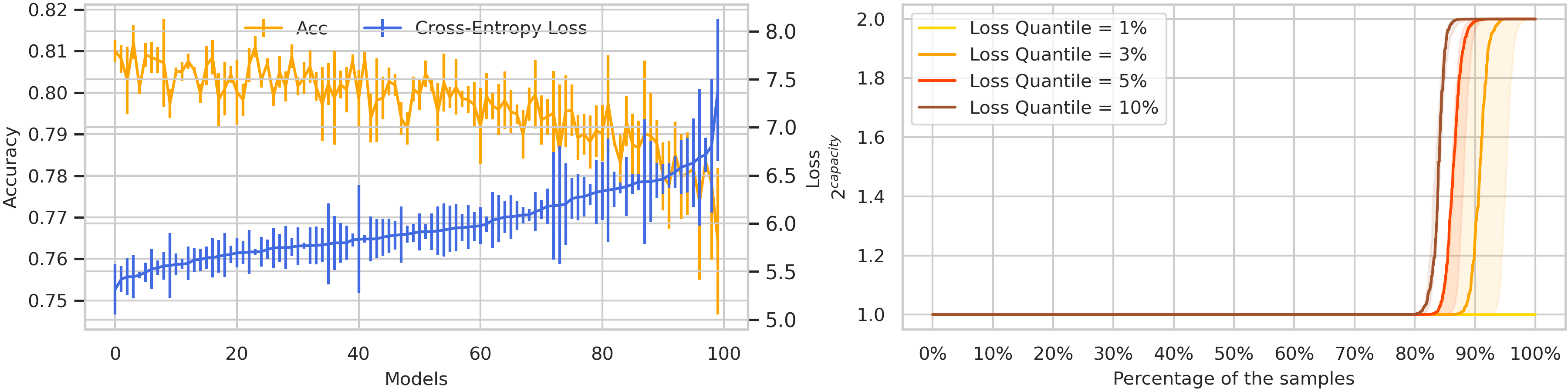}
\caption{UCI Adult dataset with logistic regression and $\ell_1$-regularization.}
\label{fig:adult-lrl1}
\end{figure}

\begin{figure}[!tb]
\centering
\includegraphics[width=.8\textwidth]{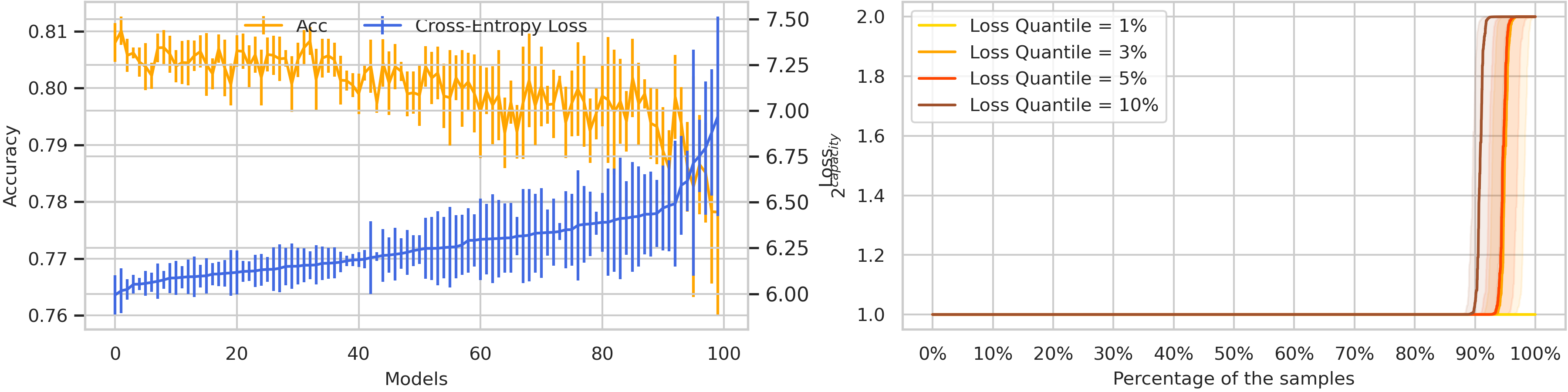}
\caption{UCI Adult dataset with logistic regression and $\ell_2$-regularization.}
\label{fig:adult-lrl2}
\end{figure}

\clearpage
\begin{figure}[!tb]
\centering
\includegraphics[width=.8\textwidth]{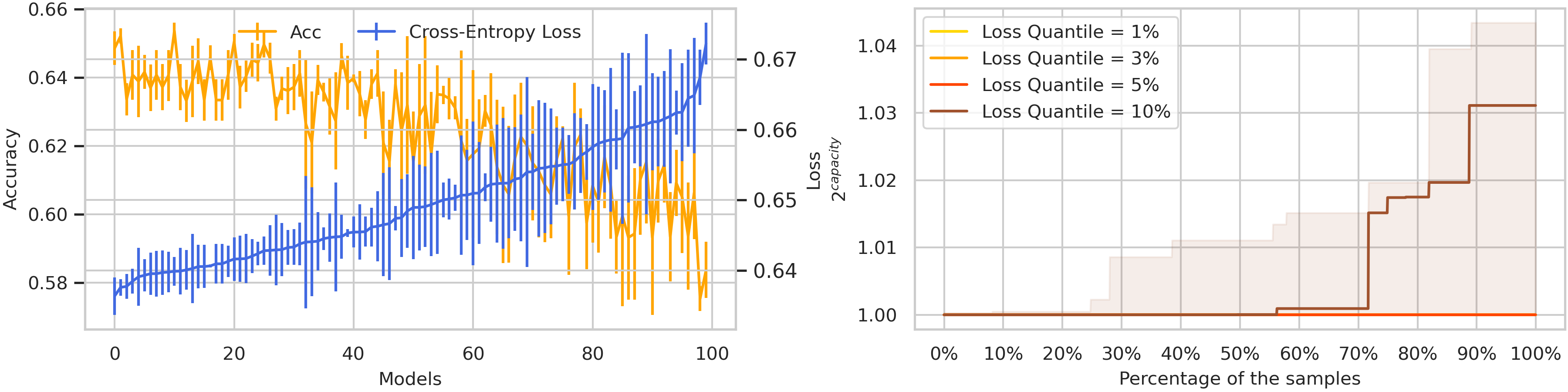}
\caption{COMPAS recidivism dataset with decision tree classifiers.}
\label{fig:compas-dt}
\end{figure}

\begin{figure}[!tb]
\centering
\includegraphics[width=.8\textwidth]{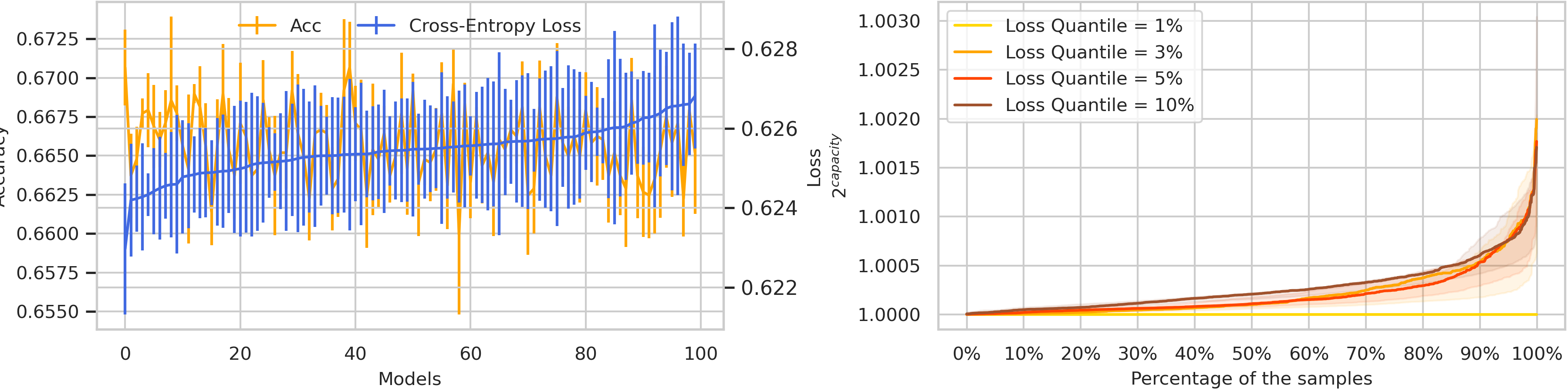}
\caption{COMPAS recidivism dataset with random forest classifiers.}
\label{fig:compas-rf}
\end{figure}

\begin{figure}[!tb]
\centering
\includegraphics[width=.8\textwidth]{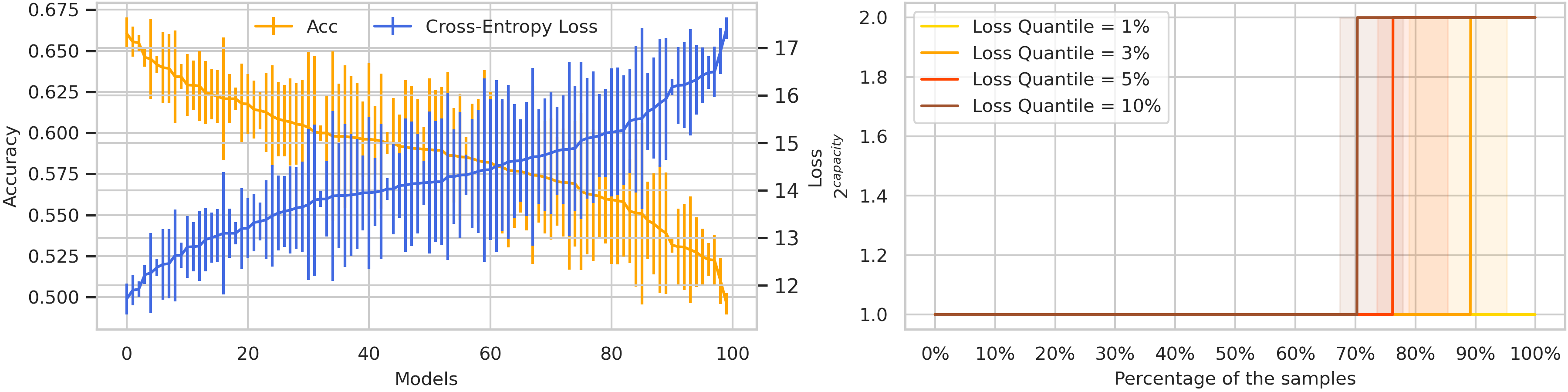}
\caption{COMPAS recidivism dataset with logistic regression and no regularization.}
\label{fig:compas-lr}
\end{figure}

\begin{figure}[!tb]
\centering
\includegraphics[width=.8\textwidth]{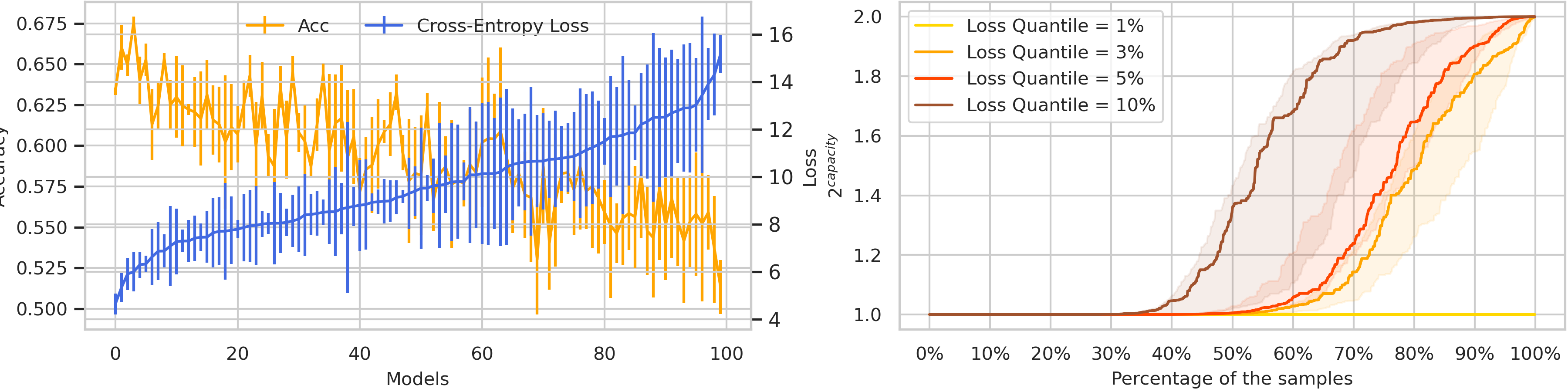}
\caption{COMPAS recidivism dataset with logistic regression and $\ell_1$-regularization.}
\label{fig:compas-lrl1}
\end{figure}

\begin{figure}[!tb]
\centering
\includegraphics[width=.8\textwidth]{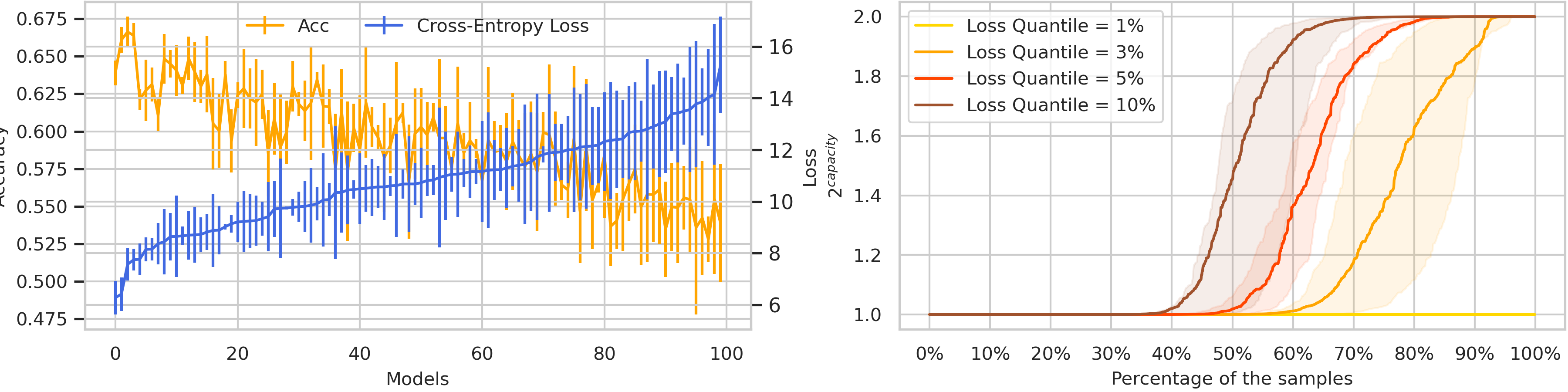}
\caption{COMPAS recidivism dataset with logistic regression and $\ell_2$-regularization.}
\label{fig:compas-lrl2}
\end{figure}

\clearpage
\begin{figure}[!tb]
\centering
\includegraphics[width=.8\textwidth]{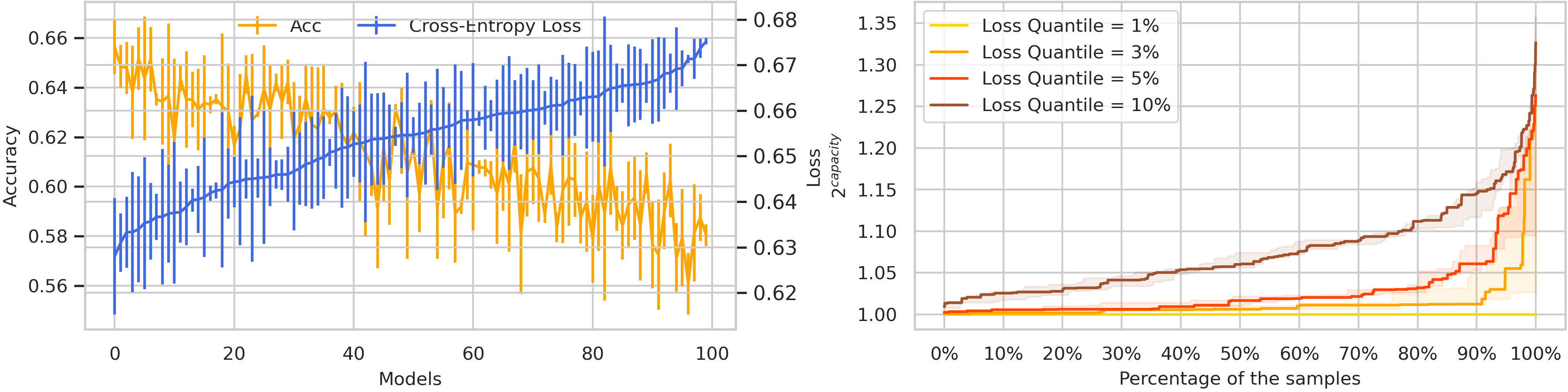}
\caption{HSLS dataset with decision tree classifiers.}
\label{fig:hsls-dt}
\end{figure}

\begin{figure}[!tb]
\centering
\includegraphics[width=.8\textwidth]{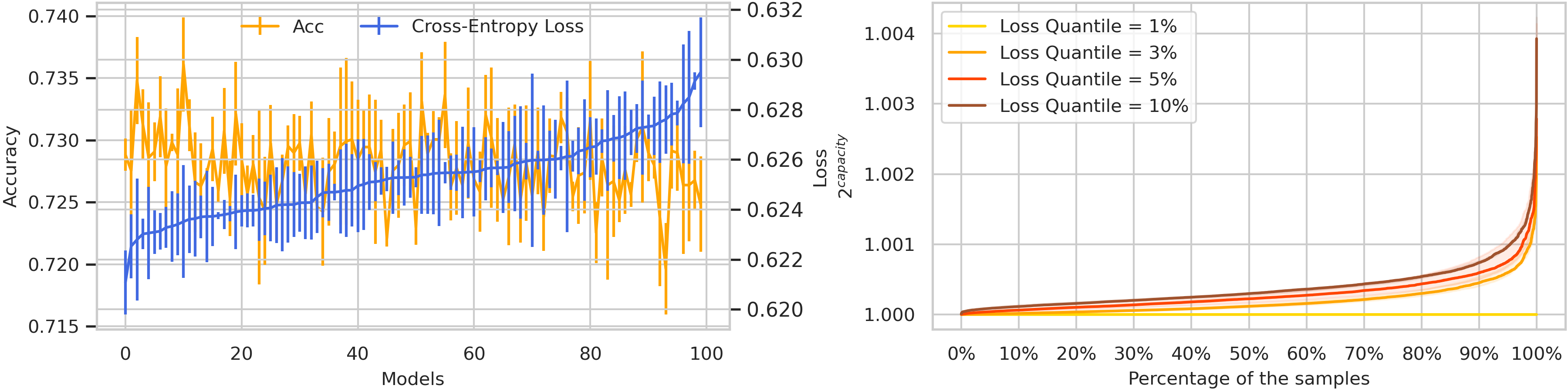}
\caption{HSLS dataset with random forest classifiers.}
\label{fig:hsls-rf}
\end{figure}

\begin{figure}[!tb]
\centering
\includegraphics[width=.8\textwidth]{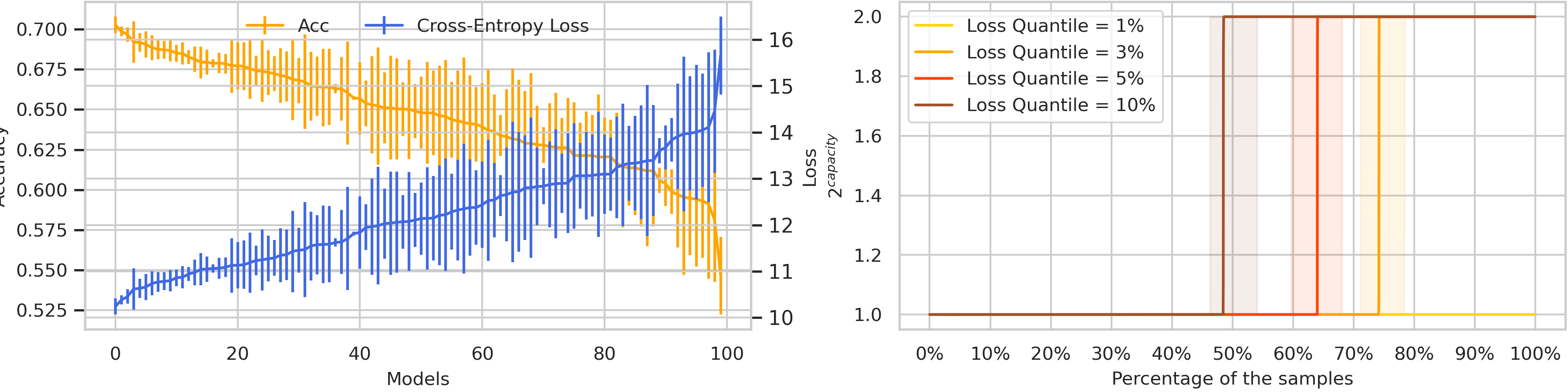}
\caption{HSLS dataset with logistic regression and no regularization.}
\label{fig:hsls-lr}
\end{figure}

\begin{figure}[!tb]
\centering
\includegraphics[width=.8\textwidth]{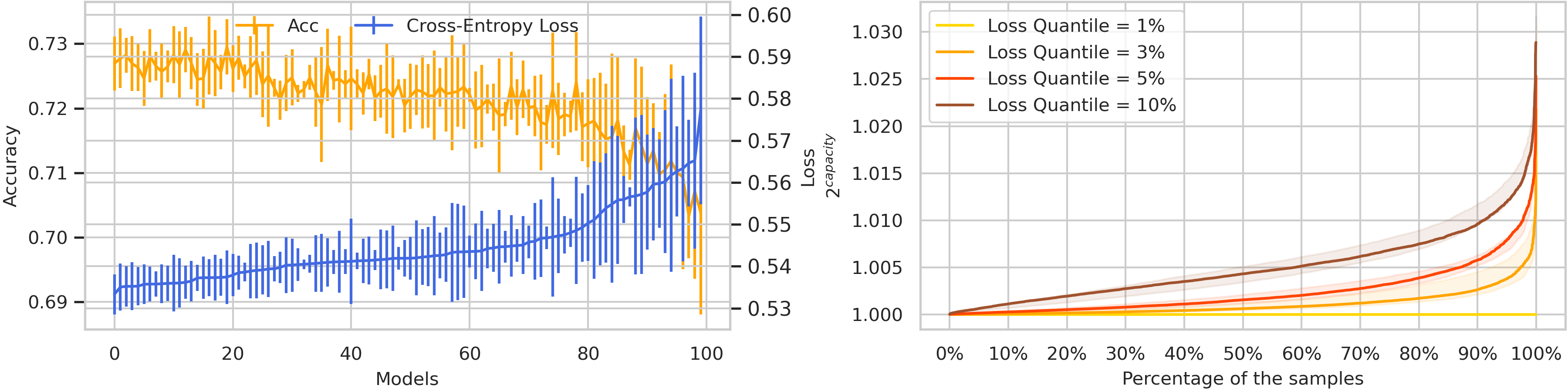}
\caption{HSLS dataset with logistic regression and $\ell_1$-regularization.}
\label{fig:hsls-lrl1}
\end{figure}

\begin{figure}[!tb]
\centering
\includegraphics[width=.8\textwidth]{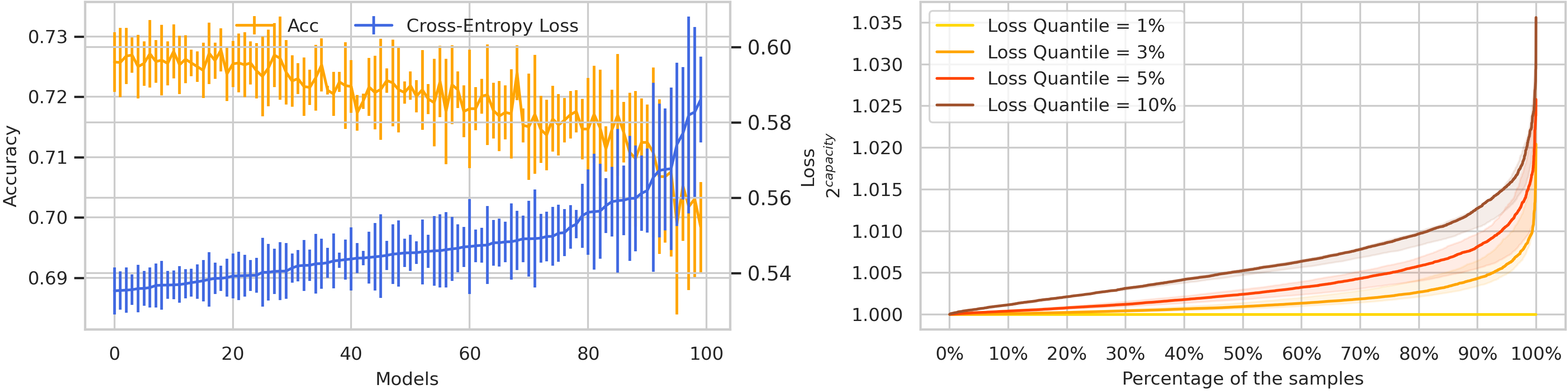}
\caption{HSLS dataset with logistic regression and $\ell_2$-regularization.}
\label{fig:hsls-lrl2}
\end{figure}

\end{document}